\documentclass[svgnames]{article}

\usepackage[a4paper, margin=1.35in]{geometry}
\usepackage[T1]{fontenc}
\usepackage{lmodern}
\usepackage{microtype}
\usepackage{amsmath,amssymb,amsthm,mathtools,mathrsfs,bm}
\usepackage{enumitem}
\usepackage{xcolor}
\usepackage{hyperref}
\usepackage{pdflscape}
\usepackage{tikz}
\usetikzlibrary{arrows.meta,positioning,fit}
\usepackage{tikz-cd}
\usepackage{tabularx}
\usepackage{booktabs}
\usepackage{authblk}
\numberwithin{equation}{section}
\usepackage[
backend=biber,
style=phys,
maxbibnames=99,
doi=true,
eprint=true,
giveninits=true
]{biblatex}
\usepackage[en-MT]{datetime2}
\DTMsetstyle{en-MT-numeric}

\definecolor{myblue}{HTML}{0251D9}
\definecolor{mygreen}{HTML}{00c71a}
\definecolor{myred}{HTML}{D90251}
\hypersetup{
	colorlinks=true,
	linkcolor=myblue,
	citecolor=mygreen,
	urlcolor=myred
}

\newtheorem{theorem}{Theorem}[section]
\newtheorem{proposition}[theorem]{Proposition}
\newtheorem{lemma}[theorem]{Lemma}
\newtheorem{corollary}[theorem]{Corollary}
\theoremstyle{definition}
\newtheorem{definition}[theorem]{Definition}
\newtheorem{assumption}[theorem]{Assumption}
\theoremstyle{remark}
\newtheorem{remark}[theorem]{Remark}

\tikzset{
	depnode/.style={
		draw,
		rounded corners,
		align=center,
		font=\scriptsize,
		inner sep=2.5pt,
		minimum width=2.6cm
	},
	assnode/.style={depnode, fill=blue!6},
	lemnode/.style={depnode, fill=green!6},
	propnode/.style={depnode, fill=orange!8},
	thmnode/.style={depnode, fill=red!6},
	note/.style={font=\scriptsize, align=left},
	dep/.style={-Latex, semithick}
}

\title{Universality of Gaussian-Mixture Reverse Kernels in Conditional Diffusion}
\date{Time of compilation: \DTMnow}

\author[1,2]{Nafiz Ishtiaque\thanks{nafiz@simis.cn}}
\author[3]{Syed Arefinul Haque}
\author[4]{Kazi Ashraful Alam}
\author[5]{Fatima Jahara}

\affil[1]{Center for Mathematics and Interdisciplinary Sciences, Fudan University, Shanghai
	200433, China}
\affil[2]{Shanghai Institute for Mathematics and Interdisciplinary Sciences (SIMIS), Shanghai
	200433, China}
\affil[3]{Northeastern University, Boston, MA 02115, USA}
\affil[4]{Infectious Diseases Division, icddr,b, Dhaka 1212, Bangladesh}
\affil[5]{Rutgers University, New Brunswick, NJ 08901, USA}

\date{}

\newcommand{\nn}{\nonumber}
\newcommand{\R}{\mathbb{R}}

\newcommand{\E}{\mathbb{E}}
\newcommand{\Prob}{\mathbb{P}}
\newcommand{\KL}{\operatorname{KL}}
\newcommand{\Leb}{\operatorname{Leb}}

\newcommand{\cX}{\mathcal{X}}
\newcommand{\cC}{\mathcal{C}}

\newcommand{\Ker}{\mathsf{Ker}}
\newcommand{\NN}{\mathrm{NN}}

\newcommand{\Comp}{\mathcal E_\mathrm{term}}
\newcommand{\push}{_*}
\newcommand{\dd}{\mathrm{d}}

\begin{document}
	\maketitle
	
	\begin{abstract}
		We prove that conditional diffusion models whose reverse kernels are
		finite Gaussian mixtures with ReLU-network logits can approximate
		suitably regular target distributions arbitrarily well in
		context-averaged conditional KL divergence, up to an irreducible
		terminal mismatch that typically vanishes with increasing diffusion horizon.
		A path-space decomposition reduces the output error to this mismatch
		plus per-step reverse-kernel errors; assuming each reverse kernel
		factors through a finite-dimensional feature map, each step becomes
		a static conditional density approximation problem, solved by
		composing Norets' Gaussian-mixture theory with quantitative ReLU
		bounds. Under exact terminal matching the resulting neural
		reverse-kernel class is dense in conditional KL.
	\end{abstract}

	\tableofcontents

	\section{Introduction}

A universality theorem asks whether an architecture is expressive enough, in
principle, to approximate the objects it is meant to learn. In classical
approximation theory, universal-approximation results for neural networks
played exactly this role, while later refinements quantified the dependence on
depth, width, and smoothness
\cite{cybenko1989approximation,hornik1989multilayer,yarotsky2017error,lu2021deep}.
This paper asks the analogous question for a class of \emph{conditional
	diffusion models}.

Diffusion methods now form a central paradigm in generative modeling
\cite{sohl2015deep,ho2020ddpm,song2020ddim,song2021sde,nichol2021improved,dhariwal2021diffusion,ho2022classifierfree,rombach2022latent,saharia2022imagen,chen2024overview}.
Yet in standard DDPM-style discretizations each reverse step is typically a
single Gaussian with learned mean and sometimes learned variance
\cite{ho2020ddpm,nichol2021improved}. Recent work argues that this ansatz can
be too restrictive, especially in few-step regimes, and proposes richer
reverse transitions based on Gaussian mixtures or related soft-mixture
denoisers \cite{li2024smd,guo2023gms,gabbur2023improvedddim}. Our aim is to
give a population-level expressivity theorem for that idea.

The target is a conditional law $c\mapsto\pi^\star(\cdot\mid c)$, i.e.\ a
Markov kernel from a context space $\cC$ to $\cX=\R^D$, with error measured by
\begin{equation}
	\mathcal E_{\KL}(\pi)
	:=
	\E_{C\sim\mathsf P_{\cC}}
	\left[\KL(\pi^\star(\cdot\mid C)\,\|\,\pi(\cdot\mid C))\right].
\end{equation}
Thus the paper concerns approximation of conditional distributions, or
equivalently Markov kernels, rather than approximation of a single
unconditional density or score field.

Our answer is affirmative for discrete-time conditional diffusions whose
reverse kernels are finite Gaussian mixtures with fixed Gaussian components and
feature-dependent logits realized by ReLU networks. The proof reduces the
output conditional-KL to a \emph{terminal mismatch} plus a sum of
\emph{per-step reverse-kernel errors}. Under a feature-sufficiency assumption,
each step factors through feature space into a static conditional density
approximation problem. At each reverse step we then use a two-level hierarchy:
first approximate the true feature-space reverse density by a
Gaussian mixture whose weights are target dependent and a priori unknowable (referred to as ``infeasible'' for that reason), then
approximate the resulting logits by ReLU networks. This yields the error
decomposition
\begin{equation*}
	\text{terminal mismatch}
	+
	\textstyle\sum_{t=1}^T
	\left(
	\text{infeasible mixture error}_t
	+
	\text{neural logit error}_t
	\right).
\end{equation*}
The mixture error is controlled through Norets' conditional
Gaussian-mixture theory \cite{norets2010}, and the logit error through the
ReLU approximation bounds of \cite{lu2021deep}. Under a smooth-logit
assumption, the resulting bounds are fully explicit.

The main conclusions are: (i) a quantitative approximation result up to
terminal mismatch (Proposition~\ref{prop:main-quant}); (ii) under exact terminal matching, a universality theorem
showing mean-KL denseness of the neural predictor class (Thm.~\ref{thm:main}); and (iii)
high-probability and almost-everywhere consequences (Corollary~\ref{cor:main-ae}).

\paragraph{Scope of the result.}
The theorem proved here is an expressivity result for a class of conditional
diffusion architectures. It does not address statistical estimation,
optimization, or score approximation. Its role is to isolate representation
error at the level of reverse-kernel parameterization, in the same spirit as
classical universal-approximation theorems for feedforward networks
\cite{cybenko1989approximation,hornik1989multilayer}. The analysis relies on
a feature-sufficiency assumption requiring each true reverse kernel to depend
on its inputs only through a known finite-dimensional feature map; this is
a modeling hypothesis, not derived from the forward process. The resulting
rates are correspondingly structural rather than practically sharp, since in
the present fully general ambient-space setting they inherit a substantial
response-side curse of dimensionality. %A fully explicit one-dimensional example -- with Gaussian target, Ornstein-Uhlenbeck (OU) forward process -- is worked out in Section~\ref{app:ou-gaussian-example}, where all error terms, constants, and network sizes are computed in closed form.

\paragraph{Relation to prior work.}
The paper sits at the intersection of three literatures. First, it is adjacent
to the theory of score-based diffusion, which studies approximation and
estimation of score fields, including recovery guarantees, minimax rates,
low-dimensional adaptation, and conditional score estimation
\cite{hyvarinen2005score,vincent2011connection,song2019generative,chen2023improved,chen2023score,oko2023minimax,fu2024conditional,tang2024conditional,chen2024overview}.
Our theorem concerns not the score class but the \emph{reverse-kernel class}
itself in discrete time. Second, it complements recent empirical and
algorithmic work advocating richer-than-Gaussian reverse transitions
\cite{li2024smd,guo2023gms,gabbur2023improvedddim} by supplying a
population-level expressivity theorem. Third, the construction draws on the
older literature on conditional density estimation and mixtures of experts:
adaptive mixtures of local experts, mixture density networks, and hierarchical
mixtures of experts provided early neural mechanisms for modeling conditional
output laws \cite{jacobs1991adaptive,bishop1994mdn,jordan1994hme};
approximation-theoretic results appeared already in work of Jiang and Tanner
\cite{jiang1999hierarchical}, and Norets later established strong
Kullback--Leibler (KL) approximation theorems for smooth mixtures of regressions
\cite{norets2010}. Our stepwise reverse model is a diffusion-era version of
that viewpoint.

In short, the paper imports the representation-theoretic viewpoint of universal
approximation into conditional diffusion, replaces the usual score-centric
perspective by a kernel-class perspective, and shows that a natural neural
Gaussian-mixture reverse architecture is universal in conditional KL once the
terminal mismatch is isolated.

\paragraph{Organization.}
\S\ref{sec:setup} formulates conditional prediction as kernel approximation
and proves the path-space KL decomposition. \S\ref{sec:gm-nn} develops the
stepwise Gaussian-mixture and neural approximation theory, including the
general explicit bounds and the sharper \(C^1\)-branch rates. \S\ref{sec:main}
states the main universality theorem and its consequences.
\S\ref{sec:conclusion} concludes. \S\ref{sec:proofs} contains some technical proofs. \S\ref{app:ou-gaussian-example} exemplifies the general formalism by approximating a conditional Gaussian kernel with Ornstein-Uhlenbeck noising, with explicit parameters and error rates.

%%%%%%%%%%%%%%%%%%%%%%%%%%%%%%%%%%%%%%%%%%%%%%%%%%%%%%%%%%%%%%%%%%%%%%%%%%%
%--------------------------------------------------------------------------
%%%%%%%%%%%%%%%%%%%%%%%%%%%%%%%%%%%%%%%%%%%%%%%%%%%%%%%%%%%%%%%%%%%%%%%%%%%

\section{Conditional prediction and diffusion kernels}\label{sec:setup}

\subsection{Prediction as kernel approximation}
For measurable spaces $(\mathcal A,\Sigma_{\mathcal A})$ and
$(\mathcal B,\Sigma_{\mathcal B})$, write $\Delta(\mathcal A)$ and
$\Ker(\mathcal A\to\mathcal B)$ for the sets of
probability measures on $\mathcal A$ and Markov kernels
$K:\mathcal A\to\Delta(\mathcal B)$ respectively.
\begin{definition}[Probabilistic prediction problem]
	Let $(\cC,\Sigma_{\cC})$ and $(\cX,\Sigma_{\cX})$ be standard Borel
	spaces, and let $\mathsf P\in\Delta(\cC\times\cX)$ with marginal
	$\mathsf P_{\cC}$ on $\cC$. By disintegration there exists a true
	conditional $\pi^\star\in\Ker(\cC\to\cX)$ with
	$\mathsf P(\dd c,\dd x)=\mathsf P_{\cC}(\dd c)\,\pi^\star(\dd x\mid c)$.
	The prediction problem is to approximate $\pi^\star$,
	$\mathsf P_{\cC}$-a.e., in KL divergence.
\end{definition}

We take $\cX=\R^D$ with Lebesgue measure $\Leb$, assume
$\pi^\star(\cdot\mid c)\ll\Leb$ for $\mathsf P_{\cC}$-a.e.\ $c$ (writing
the density as $\pi^\star(x\mid c)$), and measure error by
\begin{equation}\label{eq:output-kl}
	\mathcal E_{\KL}(\pi)
	:=
	\E_{C\sim\mathsf P_{\cC}}
	\left[\KL(\pi^\star(\cdot\mid C)\,\|\,\pi(\cdot\mid C))\right].
\end{equation}

\subsection{Forward and backward diffusion paths}\label{sec:paths}

Fix diffusion horizon $T\in\mathbb N$ and forward kernels
$q_t\in\Ker(\cX\to\cX)$, $t=0,\dots,T-1$, with strictly positive
densities $q_t(x_{t+1}\mid x_t)$.

\begin{definition}[Forward path kernel]\label{def:forwardPath}
	Define $q^{\pi^\star}\in\Ker(\cC\to\cX^{T+1})$ by
	\begin{equation}
		q^{\pi^\star}(\dd x_{0:T}\mid c)
		:=
		\pi^\star(\dd x_0\mid c)\prod_{t=0}^{T-1}q_t(\dd x_{t+1}\mid x_t),
		\label{eq:forwardPath}
	\end{equation}
	\begin{equation}\label{eq:forward-density}
		\text{with density} \qquad q^{\pi^\star}(x_{0:T}\mid c)
		=\pi^\star(x_0\mid c)\prod_{t=0}^{T-1}q_t(x_{t+1}\mid x_t).
	\end{equation}
\end{definition}

Write $q_t^{\pi^\star}(\cdot\mid c)$ for the $X_t$-marginal of
$q^{\pi^\star}(\cdot\mid c)$. Fix a reference law
$\nu\in\Delta(\cX)$ with density $\nu(x)$; this is the initialization
of the reverse chain. In practice $\nu$ is chosen so that
$q_T^{\pi^\star}(\cdot\mid c)\approx\nu$ for large $T$; we track the
mismatch explicitly.

\begin{definition}[Backward path kernel and induced predictor]\label{def:backwardPath}
	A reverse model is a family
	$p_t\in\Ker(\cC\times\cX\to\cX)$, $t=1,\dots,T$, with positive
	densities $p_t(x_{t-1}\mid c,x_t)$. It induces a path kernel
	\begin{equation}
		p^\nu(\dd x_{0:T}\mid c)
		:=\nu(\dd x_T)\prod_{t=1}^T p_t(\dd x_{t-1}\mid c,x_t),
		\label{eq:backward-density}
	\end{equation}
	with output (induced predictor)
	$\pi_{p,\nu}(\cdot\mid c):=(x_{0:T}\mapsto x_0)_*\,p^\nu(\cdot\mid c)$.
\end{definition}

\subsection{True reverse kernels and path-space KL}

\begin{definition}[Forward induced joint measures]\label{def:forwardJoint}
	Let $\mathsf Q(\dd c,\dd x_{0:T}):=\mathsf P_{\cC}(\dd c)\,q^{\pi^\star}(\dd x_{0:T}\mid c)$.
	For $t\in\{1,\dots,T\}$, let $\mathsf Q_t$ be the marginal of
	$(C,X_{t-1},X_t)$ under $\mathsf Q$:
	\begin{equation}
		\mathsf Q_t:=((c,x_{0:T})\mapsto(c,x_{t-1},x_t))\push\mathsf Q,
		\label{eq:Qprojection}
	\end{equation}
	with further marginalization
	$(\mathsf Q_t)_{C,X_t}:=((c,x_{t-1},x_t)\mapsto(c,x_t))\push\mathsf Q_t$.
	\label{eq:truePathJoint}
\end{definition}

\begin{definition}[True reverse kernel]
	By disintegration of $\mathsf Q_t$ over $(\mathsf Q_t)_{C,X_t}$,
	there exists $r_t^\star\in\Ker(\cC\times\cX\to\cX)$ with
	\begin{equation}
		\mathsf Q_t(\dd c,\dd x_{t-1},\dd x_t)
		=(\mathsf Q_t)_{C,X_t}(\dd c,\dd x_t)\,r_t^\star(\dd x_{t-1}\mid c,x_t),
		\label{eq:Qtmarginal}
	\end{equation}
	where $r_t^\star(\cdot\mid c,x_t)\ll\Leb$ for
	$(\mathsf Q_t)_{C,X_t}$-a.e.\ $(c,x_t)$.
\end{definition}

\begin{lemma}[True reverse]\label{lem:trueReverse}
	The forward and reverse factorizations of the law of $(X_{t-1},X_t)$ coincide for $\mathsf P_\cC$-a.e. $c$:
	\begin{equation}
		q_{t-1}^{\pi^\star}(\dd x_{t-1}\mid c)\,q_{t-1}(\dd x_t\mid x_{t-1})
		=q_t^{\pi^\star}(\dd x_t\mid c)\,r_t^\star(\dd x_{t-1}\mid c,x_t).
		\label{eq:forwardStep}
	\end{equation}
\end{lemma}
\begin{proof}
	Marginalizing $\mathsf Q$ to $(C,X_{t-1},X_t)$ via the Markov property of the forward chain gives
	$\mathsf Q_t(\dd c,\dd x_{t-1},\dd x_t)=\mathsf P_{\cC}(\dd c)\,q_{t-1}^{\pi^\star}(\dd x_{t-1}\mid c)\,q_{t-1}(\dd x_t\mid x_{t-1})$.
	Integrating out $x_{t-1}$ yields $(\mathsf Q_t)_{C,X_t}(\dd c,\dd x_t)=\mathsf P_{\cC}(\dd c)\,q_t^{\pi^\star}(\dd x_t\mid c)$.
	Comparing with \eqref{eq:Qtmarginal} gives \eqref{eq:forwardStep}.
\end{proof}

\begin{definition}[Terminal mismatch and path error]
	\begin{equation}
		\Comp(c):=\KL(q_T^{\pi^\star}(\cdot\mid c)\,\|\,\nu),
		\qquad
		\mathcal U_{\KL}(p,\nu)
		:=\E_{C\sim\mathsf P_{\cC}}
		[\KL(q^{\pi^\star}(\cdot\mid C)\,\|\,p^\nu(\cdot\mid C))].
		\label{eq:path-error-def}
	\end{equation}
\end{definition}

\begin{lemma}[Path-space KL decomposition]\label{lem:path-kl}
	For $\mathsf P_{\cC}$-a.e.\ $c$,
	\begin{equation}
		\KL(q^{\pi^\star}(\cdot\mid c)\,\|\,p^\nu(\cdot\mid c))
		=\Comp(c)
		+\sum_{t=1}^T
		\E_{X_t\sim q_t^{\pi^\star}(\cdot\mid c)}
		[\KL(r_t^\star(\cdot\mid c,X_t)\,\|\,p_t(\cdot\mid c,X_t))].
		\label{eq:path-kl-pointwise}
	\end{equation}
	Moreover, by data processing,
	$\KL(\pi^\star(\cdot\mid c)\,\|\,\pi_{p,\nu}(\cdot\mid c))
	\le\KL(q^{\pi^\star}(\cdot\mid c)\,\|\,p^\nu(\cdot\mid c))$,
	so
	\begin{equation}\label{eq:output-leq-path-avg}
		\mathcal E_{\KL}(\pi_{p,\nu})
		\le\E_{C\sim\mathsf P_\cC}[\Comp(C)]
		+\sum_{t=1}^T
		\E_{(C,X_t)\sim(\mathsf Q_t)_{C,X_t}}
		[\KL(r_t^\star(\cdot\mid C,X_t)\,\|\,p_t(\cdot\mid C,X_t))].
	\end{equation}
\end{lemma}
\begin{proof}
	By iterating Lemma~\ref{lem:trueReverse} from $t=T$ down to $t=1$,
	$q^{\pi^\star}(\dd x_{0:T}\mid c)
	=q_T^{\pi^\star}(\dd x_T\mid c)\prod_{t=1}^T r_t^\star(\dd x_{t-1}\mid c,x_t)$.
	Taking the log-ratio with $p^\nu(\dd x_{0:T}\mid c)=\nu(\dd x_T)\prod_t p_t(\dd x_{t-1}\mid c,x_t)$ and integrating under $q^{\pi^\star}$ gives
	\eqref{eq:path-kl-pointwise}. The marginal bound follows by data processing
	for the projection $x_{0:T}\mapsto x_0$; averaging over $C$ gives
	\eqref{eq:output-leq-path-avg}.
\end{proof}

\begin{remark}[Role of the diffusion horizon]
	\label{rem:horizon-terminal-mismatch}
	Increasing $T$ typically drives $q_T^{\pi^\star}(\cdot\mid c)$ closer
	to $\nu$, shrinking $\E_{C\sim\mathsf P_\cC}[\Comp(C)]$, but also adds more stepwise error terms.
	Thus $T$ controls a tradeoff between terminal matching and reverse-chain
	length.
\end{remark}

%%%%%%%%%%%%%%%%%%%%%%%%%%%%%%%%%%%%%%%%%%%%%%%%%%%%%%%%%%%%%%%%%%%%%%%%%%%
%--------------------------------------------------------------------------
%%%%%%%%%%%%%%%%%%%%%%%%%%%%%%%%%%%%%%%%%%%%%%%%%%%%%%%%%%%%%%%%%%%%%%%%%%%

\section{Gaussian-mixture reverse kernels with neural parameterization}
\label{sec:gm-nn}

\subsection{Stepwise feature engineering}

Assume that at step $t$ the input $(C,X_t)$ is encoded by a measurable map
$\phi_t:\cC\times\cX\to\R^{d_t}$, $Z_t:=\phi_t(C,X_t)$.
Projections and composition with $\phi_t$ induce pushforward measures from $\mathsf Q_t$:
\begin{equation}
	\begin{tikzcd}
		(c,x_{t-1},x_t) \arrow[r,mapsto, "\pi^1"] \arrow[d,mapsto, "\pi^3"] & (x_{t-1},\phi_t(c,x_t)) \arrow[d,mapsto, "\pi^2"] \\
		(c,x_t) \arrow[r,mapsto, "\pi^4"] & \phi_t(c,x_t),
	\end{tikzcd}
	\quad\text{induces}\quad
	\begin{tikzcd}
		\mathsf Q_t \arrow[r,mapsto, "\pi^1\push"] \arrow[d,mapsto, "\pi^3\push"] & F_t \arrow[d,mapsto, "\pi^2\push"] \\
		(\mathsf Q_t)_{C,X_t} \arrow[r,mapsto, "\pi^4\push"] & (F_t)_{Z_t}
	\end{tikzcd},
	\label{eq:featureMeas}
\end{equation}
where $F_t$ is the induced joint law of $(X_{t-1},Z_t)$ and $(F_t)_{Z_t}$ the induced feature space measure.

\begin{assumption}[Properties of the feature map]
	\label{ass:feature}
	For each $t\in\{1,\dots,T\}$ there exist a compact set
	$K_t\subset\R^{d_t}$ and a measurable
	$\rho_t^\star:\R^D\times K_t\to[0,\infty)$ such that:
	\begin{enumerate}[label=(\roman*)]
		\item \label{assItem:compactFeature}
		\textbf{Compact feature support.}
		$(F_t)_{Z_t}(K_t)=1$.
		
		\item \label{assItem:featureDisintegration}
		\textbf{Disintegration.}
		$F_t(\dd y,\dd z)=\rho_t^\star(y\mid z)\,\Leb(\dd y)\,(F_t)_{Z_t}(\dd z)$.
		
		\item \label{assItem:featureSufficiency}
		\textbf{Feature sufficiency.}
		For $(\mathsf Q_t)_{C,X_t}$-a.e.\ $(c,x_t)$,
		$r_t^\star(y\mid c,x_t)=\rho_t^\star(y\mid\phi_t(c,x_t))$
		for $\Leb$-a.e.\ $y$.
	\end{enumerate}
\end{assumption}

\begin{remark}\label{rem:feature-role}
	Condition~\ref{assItem:featureSufficiency} is the essential modeling
	hypothesis: the true reverse kernel depends on $(c,x_t)$ only through
	$z=\phi_t(c,x_t)$.
	Condition~\ref{assItem:featureDisintegration} is a consequence of
	\ref{assItem:featureSufficiency} plus absolute continuity of
	$r_t^\star$, but is stated for clarity.
	Condition~\ref{assItem:compactFeature} restricts the effective feature
	domain to a compact set, enabling uniform approximation of the logits. This is motivated by the intuition that in practical problems features should not be represented by truly unbounded parameters. This condition can be relaxed by allowing the support to decay and keeping track of tail error, as we demonstrate in example with Gaussian feature support in \S\ref{app:ou-gaussian-example} (see \S\ref{sec:windows} and \eqref{eq:ex-tail}). We assume strict compactness in formal proofs for simplicity.
\end{remark}

\begin{lemma}[Per-step error factors through features]
	\label{lem:feature-factorization}
	Let $\tilde p_t\in\Ker(\R^{d_t}\to\cX)$ have density
	$\tilde p_t(y\mid z)$, and define
	$d_{\KL,t}(\tilde p_t)
	:=\E_{Z_t\sim(F_t)_{Z_t}}
	[\KL(\rho_t^\star(\cdot\mid Z_t)\,\|\,\tilde p_t(\cdot\mid Z_t))]
	$.
	If the induced reverse kernel is
	$p_t(y\mid c,x_t):=\tilde p_t(y\mid\phi_t(c,x_t))$,
	then
	\begin{equation}
		d_{\KL,t}(\tilde p_t)
		=\E_{(C,X_t)\sim(\mathsf Q_t)_{C,X_t}}
		[\KL(r_t^\star(\cdot\mid C,X_t)\,\|\,p_t(\cdot\mid C,X_t))].
		\label{eq:feature-factorization-risk}
	\end{equation}
\end{lemma}
\begin{proof}
	By \ref{assItem:featureSufficiency}, the KL integrand depends on
	$(c,x_t)$ only through $\phi_t(c,x_t)$; pushing $\mathsf Q_t$ forward
	to $F_t$ via \eqref{eq:featureMeas} converts the right-hand side of \eqref{eq:feature-factorization-risk} into
	$d_{\KL,t}(\tilde p_t)$.
\end{proof}

\subsection{Regularity assumptions on forward densities}
\label{sec:regularity}
The approximation theorem \cite[Corollary~2.1]{norets2010} is formulated in terms of regularity of the target densities (reverse densities in our case). In the present prediction setting, however, it is more natural to impose assumptions on the forward objects: the target conditional $\pi^\star$ and the noising kernels $q_t$. These are the quantities that define the problem, whereas the reverse kernels are derived from them. We therefore assume regularity on the forward process and transfer it to the reverse densities through the bridge identity of Lemma~\ref{lem:trueReverse}.

\begin{assumption}[Regularity of forward densities]
	\label{ass:regularity}
	For $s \in \{0,\cdots,T-1\}$
	\begin{enumerate}[label=(\roman*)]
		\item \label{assItem:positivity}
		\textbf{Strict positivity.}
		$\pi^\star(x_0\mid c)>0$ for $\mathsf P_{\cC}$-a.e.\ $c$ and all
		$x_0$; $q_s(x_{s+1}\mid x_s)>0$ for all $s,x_s,x_{s+1}$.
		
		\item \label{assItem:continuity}
		\textbf{Continuity and local integrable dominance.}
		$x_0\mapsto\pi^\star(x_0\mid c)$ is continuous for
		$\mathsf P_{\cC}$-a.e.\ $c$; $u\mapsto q_s(u\mid v)$ is
		continuous for every $v$; and for every compact $K\subset\cX$,
		$\sup_{u\in K}q_s(u\mid v)\le g_{s,c,K}(v)$ for some
		$g_{s,c,K}\in L^1(q_s^{\pi^\star}(\cdot\mid c))$.
		
		\item \label{assItem:finiteM2}
		\textbf{Finite second moments.}
		$\int_\cC\int_{\R^D}\|y\|^2\,q_s^{\pi^\star}(\dd y\mid c)\,\mathsf P_{\cC}(\dd c)<\infty$.
		
		\item \label{assItem:logBound}
		\textbf{Integrable local log-variation bounds.}
		There exist hypercubes $C_s(x)\ni x$ of side-length $r_s>0$ and functions
		$a_{s-1}:\cX\to[0,\infty)$,
		$b_s:\cX\times\cX\to[0,\infty)$
		such that: for $\mathsf P_{\cC}$-a.e.\ $c$,
		and
		\begin{align}
			\text{for a.e. } x_s,&&
			\log\frac{q_s^{\pi^\star}(x_s\mid c)}
			{\inf_{y\in C_s(x_s)}
				q_s^{\pi^\star}(y\mid c)}
			\le&\; a_{s-1}(x_s);
			\label{eq:marginal-oscillation}
		\\
			\text{for a.e. } (x_s,x_{s+1}),&&
			\log\frac{q_s(x_{s+1}\mid x_s)}
			{\inf_{y\in C_s(x_s)}q_s(x_{s+1}\mid y)}
			\le&\; b_s(x_s,x_{s+1});
			\label{eq:conditioning-variable-oscillation}
		\end{align}
		and
		$\E_{(C,X_s,X_{s+1})\sim\mathsf Q_{s+1}}
		[a_{s-1}(X_s)+b_s(X_s,X_{s+1})]<\infty$.
	\end{enumerate}
\end{assumption}

\begin{proposition}[Step-$t$ target satisfies the $\mathcal M_0$ conditions
	of Norets {\cite[Assumption~2.1]{norets2010}}]
	\label{prop:noretsM0}
	Write $f_t(y\mid z):=\rho_t^\star(y\mid z)$.  Under
	Assumptions~\ref{ass:feature}\ref{assItem:featureDisintegration}\ref{assItem:featureSufficiency} and~\ref{ass:regularity}:
	\begin{enumerate}[label=(\roman*)]
		\item\label{item:noretsM0-pos}
		$y\mapsto f_t(y\mid z)$ is strictly positive and continuous for
		$(F_t)_{Z_t}$-a.e.\ $z$;
		\item\label{item:noretsM0-M2}
		$\int\|y\|^2\,F_t(\dd y,\dd z)<\infty$;
		\item\label{item:noretsM0-log}
		$\int
		\log\frac{f_t(y\mid z)}{\inf_{u\in C_{t-1}(y)}f_t(u\mid z)}\,
		F_t(\dd y,\dd z)<\infty$.
	\end{enumerate}
\end{proposition}
This is proved in Appendix~\ref{sec:proof-prop-noretsM0}. 

This proposition is what allows us to invoke Norets' theory at step $t$, and hence to treat the infeasible Gaussian-mixture class $\mathcal M_0$ as a valid first approximation to the true reverse density. The proposition along with Assumption~\ref{ass:regularity} are checked explicitly for the example in \S\ref{app:ou-gaussian-example} (see \S\ref{sec:verifyForwardAss}).

\subsection{Stepwise model hierarchy: infeasible and neural reverse kernels}
\label{sec:models}

We set up the approximation hierarchy
$f_t\leadsto\tilde p_t^{\mathcal M_0,m_t}\leadsto\tilde p_{t,\vartheta_t}^{\NN}$,
\label{eq:modelHierarchy}
where the first arrow is the Norets' infeasible Gaussian-mixture approximation with target dependent mixing coefficients
\cite[\S2]{norets2010} and the second replaces the exact coefficients by ReLU outputs.

\begin{definition}[Step-$t$ partition data and cell probabilities]
	\label{def:partitionX}
	For each $m_t\ge 1$, let $\{A_{t,j}^{m_t}\}_{j=0}^{m_t}$ partition
	$\R^D$: cells $A_{t,1}^{m_t},\dots,A_{t,m_t}^{m_t}$ are adjacent
	hypercubes of side-length $h_{t,m_t}\to 0$ with centers
	$\mu_{t,j}^{m_t}$; $A_{t,0}^{m_t} := \mathbb R^D \backslash \bigcup_{j=1}^{m_t} A_{t,j}^{m_t}$ is the remainder.
	There exist $\delta_{t,m_t}\to 0$ such that the hypercube $C_{\delta_{t,m_t}}(y)$ of side-length $\delta_{t,m_t}$ and centered at $y$ eventually misses $A_{t,0}^{m_t}$ for every $y$. Scales
	$\sigma_{t,m_t},\sigma_{t,0}>0$ satisfy
	\begin{equation}
		\frac{\sigma_{t,m_t}}{\delta_{t,m_t}}\to 0,
		\quad
		\frac{\delta_{t,m_t}^{D-1}h_{t,m_t}}{\sigma_{t,m_t}^D}\to 0,
		\quad
		\frac{(r_{t-1}/2)^D}{(2\pi\sigma_{t,0}^2)^{D/2}}\le 2^{-(D+1)}.
		\label{eq:delta-sigma-h-condition}
	\end{equation}
	Define cell probabilities:
	$G_{t,j}^{m_t}(z):=\int_{A_{t,j}^{m_t}}\rho_t^\star(y\mid z)\,\dd y$.
	\label{eq:cellProb}
\end{definition}

%\begin{assumption}[Partition regularity]
%	\label{ass:cell-regularity}
%	For every $m_t,j$, the map $G_{t,j}^{m_t}:K_t\to(0,1)$ is continuous.
%\end{assumption}

\paragraph{Baseline-softmax family.}
For logits $u=(u_1,\dots,u_{m_t}):K_t\to\R^{m_t}$, define
\begin{equation}
	\alpha_{t,j}[u](z)
	:=\frac{\exp(u_j(z))}{1+\sum_k\exp(u_k(z))}
	\;\;(j\ge 1),
	\qquad
	\alpha_{t,0}[u](z)
	:=\frac{1}{1+\sum_k\exp(u_k(z))}.
	\label{eq:generic-softmax}
\end{equation}
The corresponding Gaussian-mixture density is
\begin{equation}
	\tilde p_t[u](y\mid z)
	:=\sum_{j=1}^{m_t}\alpha_{t,j}[u](z)\,
	\varphi(y;\mu_{t,j}^{m_t},\sigma_{t,m_t}^2I_D)
	+\alpha_{t,0}[u](z)\,\varphi(y;0,\sigma_{t,0}^2I_D),
	\label{eq:generic-logit-mixture}
\end{equation}
with induced kernel
$p_t[u](\dd y\mid c,x_t):=\tilde p_t[u](\dd y\mid\phi_t(c,x_t))$.
\label{eq:generic-logit-kernel}

\paragraph{Exact log-odds.}
Define
$\Lambda_{t,j}^{m_t}(z):=\log(G_{t,j}^{m_t}(z)/G_{t,0}^{m_t}(z))$
for $j=1,\dots,m_t$, so that
$G_{t,j}^{m_t}=\alpha_{t,j}[\Lambda_t^{m_t}]$ by \eqref{eq:generic-softmax}.

We now introduce the two model classes that form the approximation hierarchy at step $t$: the infeasible exact-logit Gaussian mixture $\mathcal M_0$, and its neural realization obtained by replacing the exact logits with a ReLU network.

\begin{definition}[Infeasible cell mixture ($\mathcal M_0$) at step $t$]
	\label{def:step-M0}
	$\tilde p_t^{\mathcal M_0,m_t}(y\mid z):=\tilde p_t[\Lambda_t^{m_t}](y\mid z)$,
	with induced kernel $p_t^{\mathcal M_0,m_t}:=p_t[\Lambda_t^{m_t}]$.
\end{definition}

\begin{definition}[Neural reverse model at step $t$]
	\label{def:neural-step-model}
	Let $a_t^{\vartheta_t}:\R^{d_t}\to\R^{m_t}$ be a ReLU network. Then
	$\tilde p_{t,\vartheta_t}^{\NN}(y\mid z):=\tilde p_t[a_t^{\vartheta_t}](y\mid z)$,
	with induced kernel $p_{t,\vartheta_t}^{\NN}:=p_t[a_t^{\vartheta_t}]$.
\end{definition}

\subsection{Explicit error rates}
\label{sec:nn-rates}

We now turn to quantitative error control for the two step-$t$ models introduced above. The first ingredient is a stability estimate showing that small uniform perturbations of the logits produce only a controlled increase in the stepwise KL error.

\begin{lemma}[Baseline-softmax stability]
	\label{lem:baseline-softmax-stability}
	For two logit vectors $u,v:K_t\to\R^{m_t}$,
	if $\max_j\sup_{z\in K_t}|u_j(z)-v_j(z)|\le\eta$, then
		\begin{equation}
			e^{-2\eta}\tilde p_t[v]\le\tilde p_t[u]\le e^{2\eta}\tilde p_t[v] \label{eq:density-bound}
		\end{equation}
		pointwise, and
		\begin{equation}
			d_{\KL,t}(\tilde p_t[u])\le d_{\KL,t}(\tilde p_t[v])+2\eta. \label{eq:kl-stability-bound}
		\end{equation}
\end{lemma}
\begin{proof}
	With $Z_u:=1+\sum_{k=1}^{m_t} e^{u_k}$ and $Z_v:=1+\sum_{k=1}^{m_t} e^{v_k}$, the uniform
	bound gives $e^{-\eta}Z_v\le Z_u\le e^\eta Z_v$. For $j\ge 1$,
	$\alpha_{t,j}[u]/\alpha_{t,j}[v]=e^{u_j-v_j}Z_v/Z_u\in[e^{-2\eta},e^{2\eta}]$;
	for $j=0$, $\alpha_{t,0}[u]/\alpha_{t,0}[v]=Z_v/Z_u\in[e^{-\eta},e^{\eta}]$.
	Summing $\alpha_{t,j}[u]$ and $\alpha_{t,j}[v]$ against the common Gaussian components (as in \eqref{eq:generic-logit-mixture}) yields the density sandwich \eqref{eq:density-bound} which implies 
	$\log(f_t/\tilde p_t[u])
	=\log(f_t/\tilde p_t[v])+\log(\tilde p_t[v]/\tilde p_t[u])
	\le\log(f_t/\tilde p_t[v])+2\eta$. Now
	integrating against $F_t$ we obtain
	\eqref{eq:kl-stability-bound}.
\end{proof}

\paragraph{Infeasible model bound.}
Define the $\delta$-boundary event
$\mathcal B_{t,m_t}(\delta):=\{(z,y):C_\delta(y)\cap A_{t,0}^{m_t}\neq\emptyset\}$,
\label{eq:Btm-def}
the local-oscillation operator
\begin{equation}
	\mathcal D_r f_t(y\mid z)
	:=\begin{cases}
		\log\frac{f_t(y\mid z)}{\inf_{u\in C_r(y)}f_t(u\mid z)},
		&\text{general},\\
		\frac{r\sqrt D}{2}\sup_{u\in C_r(y)}\|\nabla_u\log f_t(u\mid z)\|,
		&f_t(\cdot\mid z)\in C^1,
	\end{cases}
	\label{eq:Duft}
\end{equation}
and
\begin{equation}
	\begin{gathered}
		L_t(\delta) :=\textstyle\int\mathcal D_\delta f_t(y\mid z)\,\dd F_t(\dd y, \dd z), \qquad
		T_{t,m_t}(\delta) :=\textstyle\int_{\mathcal B_{t,m_t}(\delta)}\mathcal D_{r_{t-1}}f_t(y\mid z)\,\dd F_t(\dd y, \dd z),\\
		U_{t,m_t}(\delta) :=\textstyle\int_{\mathcal B_{t,m_t}(\delta)}
		\left(\tfrac{\|y\|^2}{2\sigma_{t,0}^2}
		-\log\tfrac{(r_{t-1}/2)^D}{(2\pi\sigma_{t,0}^2)^{D/2}}\right)\dd F_t(\dd y, \dd z).
	\end{gathered}
	\label{eq:LTU}
\end{equation}
The explicit $\mathcal M_0$ bound is
\begin{equation}
	B_{t,m_t}^{\mathrm{Nor}}
	:=L_t(\delta_{t,m_t})+T_{t,m_t}(\delta_{t,m_t})+U_{t,m_t}(\delta_{t,m_t})
	+\frac{6D^{3/2}\delta_{t,m_t}^{D-1}h_{t,m_t}}{(2\pi)^{D/2}\sigma_{t,m_t}^D}
	+2e^{-(\delta_{t,m_t}/\sigma_{t,m_t})^2/8}.
	\label{eq:Bnorets}
\end{equation}

\begin{proposition}[Explicit $\mathcal M_0$ bound]
	\label{prop:M0-explicit-bound}
	Under Assumptions~\ref{ass:feature}\ref{assItem:featureDisintegration}\ref{assItem:featureSufficiency} and \ref{ass:regularity}, for all large $m_t$,
	$d_{\KL,t}(\tilde p_t^{\mathcal M_0,m_t})\le B_{t,m_t}^{\mathrm{Nor}}$. The bound converges to zero as $m \to \infty$.
\end{proposition}
\begin{proof}
	Proposition~\ref{prop:noretsM0} verifies \cite[Assumption~2.1]{norets2010}.
	%Assumption~\ref{ass:cell-regularity} provides \cite[Assumption~3.1]{norets2010}. 
	Then the bound and its limit are given by \cite[Corollary~2.1(i,ii)]{norets2010}.
\end{proof}

\begin{corollary}[Power rate under the $C^1$ branch]
	\label{cor:M0-power-rate}
	If additionally $f_t(\cdot\mid z)\in C^1$, there exist
	$\theta_t>2$ with
	$\int|y_i|^{\theta_t}\dd F_t<\infty$ for $i \in \{1,\cdots,D\}$, and $i_t\in\{1,\dots,D\}$ for which
	$\int|y_{i_t}|^{\theta_t-2}\sup_{u\in C_{r_{t-1}}(y)}\|\nabla_u\log f_t\|\,\dd F_t\allowbreak<\infty$,
	\label{eq:q-moment}
	then for every $\xi_t>0$ there exists $C_t^{\mathrm{Nor}}<\infty$ with
	\begin{equation}
		d_{\KL,t}(\tilde p_t^{\mathcal M_0,m_t})
		\le C_t^{\mathrm{Nor}}\,m_t^{-\beta_t},
		\qquad
		%\beta_t:=\frac{1}{D(2+\frac{1}{\theta_t-2}+\xi_t)}.
		\beta_t:=D^{-1}/\left(2+\frac{1}{\theta_t-2}+\xi_t\right).
		\label{eq:M0-power-rate}
	\end{equation}
\end{corollary}
\begin{proof}
	This is \cite[Corollary~2.1(iii)]{norets2010}.
\end{proof}

\begin{remark}[Ambient versus intrinsic target dimension]
	\label{rem:ambient-vs-intrinsic-dimension}
	The rate $\beta_t=O(1/D)$ reflects the curse of dimensionality on the
	response side of the Norets-type mixture approximation: the bounds are
	controlled by the ambient dimension of $X_{t-1}\in\R^D$, not by the
	feature dimension $d_t$. In many applications, however, the reverse law
	of $X_{t-1}$ given $(C,X_t)$ may effectively vary along far fewer than
	$D$ directions. If at step $t$ the true reverse kernel admits a
	suitably controlled factorization through a latent space of dimension
	$\mathfrak d_t\ll D$, one would expect $\mathfrak d_t$ to replace $D$
	in the response-side bounds. Establishing such a reduction requires
	additional geometric or probabilistic structure beyond our assumptions,
	and will not hold in general. In more restricted cases, it is an interesting open problem to reliably detect such low-dimensional structure, especially in view of recent
	manifold-adaptive analyses on the score-based side
	\cite{chen2023score,tang2024conditional}. We work throughout with
	ambient-dimensional rates, noting that they may be greatly pessimistic
	when the reverse dynamics is effectively low-dimensional.
\end{remark}

\paragraph{Neural approximation of log-odds.}
Fix a closed cube $Q_t\supset K_t$ and let
$\Psi_t:[0,1]^{d_t}\to Q_t$ be affine.

\begin{assumption}[Smooth log-odds]
	\label{ass:smooth-logodds}
	For every $m_t,j$, the log-odds $\Lambda_{t,j}^{m_t}$ extend to
	$C^{s_t}$ on a neighborhood of $Q_t$, and
	$\Lambda_{t,m_t}^{\max}
	:=\max_j\|\Lambda_{t,j}^{m_t}\circ\Psi_t\|_{C^{s_t}([0,1]^{d_t})}<\infty$.
\end{assumption}

\begin{remark}\label{rem:smooth-logodds-when}
	A sufficient condition is that $\partial_z^\alpha\rho_t^\star(y\mid z)$
	exist for any multi-index $\alpha$ with $|\alpha|\le s_t$, be jointly continuous, and admit an
	integrable envelope for differentiation under the integral defining the cell probabilities in \eqref{eq:cellProb}. For Gaussian forward kernels these conditions are
	typically straightforward, e.g., in the Gaussian example of Section~\ref{app:ou-gaussian-example} (see~\eqref{eq:ex-Lambda-growth}).
\end{remark}

\begin{proposition}[ReLU approximation of log-odds]
	\label{prop:nn-logit-rate}
	Under Assumption~\ref{ass:smooth-logodds}, for any $\mathscr N_t,\mathscr L_t\ge 1$
	there exists a ReLU network
	$a_t^{\vartheta_t}:\R^{d_t}\to\R^{m_t}$ with width
	$\mathcal W_t\le m_t\,17s_t^{d_t+1}3^{d_t}d_t(\mathscr N_t+2)\log_2(8\mathscr N_t)$,
	depth
	$\mathcal L_t\le 18s_t^2(\mathscr L_t+2)\log_2(4\mathscr L_t)+2d_t$,
	and uniform error
	\begin{equation}
		\eta_t^{\NN}
		:=\max_j\sup_{z\in K_t} \left|a_{t,j}^{\vartheta_t}(z)-\Lambda_{t,j}^{m_t}(z)\right|
		\le 85(s_t+1)^{d_t}8^{s_t}\Lambda_{t,m_t}^{\max}
		\mathscr N_t^{-2s_t/d_t}\mathscr L_t^{-2s_t/d_t}.
		\label{eq:nn-logit-error}
	\end{equation}
\end{proposition}
\begin{proof}
	Apply \cite[Theorem~1.1]{lu2021deep} to each
	$\Lambda_{t,j}^{m_t}\circ\Psi_t\in C^{s_t}([0,1]^{d_t})$, compose
	with $\Psi_t^{-1}$, take the max over $j$, and stack in parallel to get the factor of $m_t$ in the upper bound of $\mathcal W_t$.
\end{proof}

\begin{proposition}[Explicit stepwise neural rate]
	\label{prop:step-neural-rate}
	Under the hypotheses of Propositions~\ref{prop:M0-explicit-bound}
	and~\ref{prop:nn-logit-rate},
	\begin{equation}
		d_{\KL,t}(\tilde p_{t,\vartheta_t}^{\NN})
		\le
		B_{t,m_t}^{\mathrm{Nor}}
		+
		170(s_t+1)^{d_t}8^{s_t}\,
		\Lambda_{t,m_t}^{\max}\,
		\mathscr N_t^{-2s_t/d_t}\mathscr L_t^{-2s_t/d_t}.
		\label{eq:nn-step-bound-general}
	\end{equation}
	If, in addition, the hypotheses of Corollary~\ref{cor:M0-power-rate}
	hold, then
	\begin{equation}
		d_{\KL,t}(\tilde p_{t,\vartheta_t}^{\NN})
		\le
		C_t^{\mathrm{Nor}}\,m_t^{-\beta_t}
		+
		170(s_t+1)^{d_t}8^{s_t}\,
		\Lambda_{t,m_t}^{\max}\,
		\mathscr N_t^{-2s_t/d_t}\mathscr L_t^{-2s_t/d_t}.
		\label{eq:nn-step-bound-power}
	\end{equation}
\end{proposition}

\begin{proof}
	Apply Lemma~\ref{lem:baseline-softmax-stability} with
	$v=\Lambda_t^{m_t}$ and $u=a_t^{\vartheta_t}$ to obtain:
	$
	d_{\KL,t}(\tilde p_{t,\vartheta_t}^{\NN})
	\le
	d_{\KL,t}(\tilde p_t^{\mathcal M_0,m_t})+2\eta_t^{\NN}.
	$
	Substituting the $\mathcal M_0$ bound from
	Proposition~\ref{prop:M0-explicit-bound} together with
	\eqref{eq:nn-logit-error} gives \eqref{eq:nn-step-bound-general}. If the
	hypotheses of Corollary~\ref{cor:M0-power-rate} also hold, then
	substituting its bound in place of Proposition~\ref{prop:M0-explicit-bound}
	gives \eqref{eq:nn-step-bound-power}.
\end{proof}

We now sum along the reverse chain.

\begin{theorem}[Total diffusion error up to terminal mismatch]
	\label{thm:total-neural-M1}
	Let
	$
	p^{\NN}:=\{p_{t,\vartheta_t}^{\NN}\}_{t=1}^T
	$
	be a neural reverse model, and let $\pi_{p^{\NN},\nu}$ denote its induced
	predictor. If the hypotheses of Proposition~\ref{prop:step-neural-rate}
	hold for each $t$, then
	\begin{equation}
		\mathcal E_{\KL}(\pi_{p^{\NN},\nu})
		\le
		\E[\Comp(C)]
		+
		\sum_{t=1}^T
		\Bigl(
		B_{t,m_t}^{\mathrm{Nor}}
		+
		170(s_t+1)^{d_t}8^{s_t}\,
		\Lambda_{t,m_t}^{\max}\,
		\mathscr N_t^{-2s_t/d_t}\mathscr L_t^{-2s_t/d_t}
		\Bigr).
		\label{eq:global-kl-explicit}
	\end{equation}
	If, in addition, the hypotheses of
	Corollary~\ref{cor:M0-power-rate} hold for each $t$, then
	\begin{equation}
		\mathcal E_{\KL}(\pi_{p^{\NN},\nu})
		\le
		\E[\Comp(C)]
		+
		\sum_{t=1}^T
		\Bigl(
		C_t^{\mathrm{Nor}}\,m_t^{-\beta_t}
		+
		170(s_t+1)^{d_t}8^{s_t}\,
		\Lambda_{t,m_t}^{\max}\,
		\mathscr N_t^{-2s_t/d_t}\mathscr L_t^{-2s_t/d_t}
		\Bigr).
		\label{eq:global-kl-power}
	\end{equation}
\end{theorem}

\begin{proof}
	Lemma~\ref{lem:path-kl} and Lemma~\ref{lem:feature-factorization} give
	$
	\mathcal E_{\KL}(\pi_{p^{\NN},\nu})
	\le
	\E[\Comp(C)]
	+
	\sum_{t=1}^T d_{\KL,t}(\tilde p_{t,\vartheta_t}^{\NN}).
	$
	Summing the bounds from Proposition~\ref{prop:step-neural-rate}
	proves \eqref{eq:global-kl-explicit} and \eqref{eq:global-kl-power}.
\end{proof}

\begin{remark}[Fully explicit bounds and the role of log-odds complexity]
	\label{rem:eps-smooth-rate}
	All terms in \eqref{eq:global-kl-explicit} and
	\eqref{eq:global-kl-power} are explicit.  The only error not
	removable by increasing the mixture resolution and network size is the
	terminal mismatch $\E[\Comp(C)]$; in particular, exact universality
	requires $q_T^{\pi^\star}(\cdot\mid c)=\nu$ for
	$\mathsf P_{\cC}$-a.e.\ $c$.
	
	The log-odds complexity $\Lambda_{t,m_t}^{\max}$ may grow with the
	partition resolution~$m_t$.  If
	$\Lambda_{t,m_t}^{\max}=O(m_t^\gamma)$ for some $\gamma\ge 0$, then
	the network sizes can be chosen so that the neural term is
	$O(m_t^{-1})$, and the overall rate is governed by the Norets term:
	qualitatively by $B_{t,m_t}^{\mathrm{Nor}}\to 0$ in the general
	branch, and explicitly by
	$C_t^{\mathrm{Nor}}m_t^{-\beta_t}$ under the $C^1$-branch
	hypotheses.  This mechanism is instantiated in the Gaussian--OU
	example of Section~\ref{app:ou-gaussian-example}
	(see~\eqref{eq:ex-network-size}).
\end{remark}

The general bound \eqref{eq:global-kl-explicit} is sufficient for qualitative
approximation and universality. The stronger \(C^1\)-branch hypotheses are used
only to obtain the explicit algebraic-rate refinement \eqref{eq:global-kl-power}.

%%%%%%%%%%%%%%%%%%%%%%%%%%%%%%%%%%%%%%%%%%%%%%%%%%%%%%%%%%%%%%%%%%%%%%%%%%%
%--------------------------------------------------------------------------
%%%%%%%%%%%%%%%%%%%%%%%%%%%%%%%%%%%%%%%%%%%%%%%%%%%%%%%%%%%%%%%%%%%%%%%%%%%

\section{Main universality theorem and consequences}
\label{sec:main}

Define the family of neural reverse models of horizon $T$
\begin{equation}
	\mathfrak P_T^{\NN}(\nu)
	:=
	\left\{
	\pi_{p^{\NN},\nu}
	\;:\;
	p^{\NN}=\{p_{t,\vartheta_t}^{\NN}\}_{t=1}^T
	\text{ is a neural reverse model of horizon }T
	\right\}.
\end{equation}

We first record the quantitative approximation statement up to terminal mismatch.
A separate zero-mismatch universality statement will then follow by imposing
exact terminal matching.

\begin{proposition}[Quantitative approximation up to terminal mismatch]
	\label{prop:main-quant}
	Fix $T\ge 1$.
	\begin{enumerate}[label=(\roman*)]
		\item\label{item:main-quant-general}
		\textbf{General case.}
		Assume that, for every $t\in\{1,\dots,T\}$,
		Assumptions~\ref{ass:feature}, \ref{ass:regularity}, and~\ref{ass:smooth-logodds} hold at step~$t$.
		Then for every $\varepsilon>0$ there exists a neural reverse model
		$p^{\NN}=\{p_{t,\vartheta_t}^{\NN}\}_{t=1}^T$ such that
		\begin{equation}
			\mathcal E_{\KL}\left(\pi_{p^{\NN},\nu}\right)
			\le
			\E[\Comp(C)]+\varepsilon,
			\quad \text{in particular,} \quad
			\inf_{\pi\in\mathfrak P_T^{\NN}(\nu)}
			\mathcal E_{\KL}(\pi)
			\le
			\E[\Comp(C)].
			\label{eq:main-general-inf}
		\end{equation}
		
		\item\label{item:main-quant-C1}
		\textbf{Explicit rate under the $C^1$ branch.}
		Assume, in addition to these assumptions,
		that for each $t\in\{1,\dots,T\}$:
		\begin{itemize}
			\item a closed cube $Q_t\subset\R^{d_t}$ with $K_t\subset Q_t$
			and an integer $s_t\ge 1$ are fixed;
			\item the $C^1$-branch hypotheses of
			Corollary~\ref{cor:M0-power-rate} hold at step~$t$ (with exponent
			$\theta_t>2$, coordinate $i_t$, and parameter $\xi_t>0$).
		\end{itemize}
		Let
		$
		\beta_t
		:=
		D^{-1}/\left(2+\frac{1}{\theta_t-2}+\xi_t\right)
		$,
		and let $C_t^{\mathrm{Nor}}<\infty$ be the constant given by
		Corollary~\ref{cor:M0-power-rate}.
		Then there exist an integer $n_0\ge 1$ and a sequence of neural
		reverse models
		$
		p^{\NN,n}
		=
		\bigl\{
		p_{t,\vartheta_t^{(n)}}^{\NN}
		\bigr\}_{t=1}^T$,
		such that, for every $n\ge n_0$,
		\begin{align}
			&\mathcal E_{\KL}\left(\pi_{p^{\NN,n},\nu}\right)
			\le
			\E[\Comp(C)]
			+
			\sum_{t=1}^T C_t^{\mathrm{Nor}}\,n^{-\beta_t}
			+
			\frac{1}{2n};
			\label{eq:main-finite-n-bound}
		\\
			\text{in particular,} \qquad &\limsup_{n\to\infty}
			\mathcal E_{\KL}\left(\pi_{p^{\NN,n},\nu}\right)
			\le
			\E[\Comp(C)].
			\label{eq:main-limsup-bound}
		\end{align}
	\end{enumerate}
\end{proposition}

This is proved in Appendix~\ref{sec:proof-main-quant}. The bound~\eqref{eq:main-finite-n-bound} is instantiated explicitly in example in \eqref{eq:ex-final}.

The preceding proposition is the main quantitative approximation result, stated explicitly up to the terminal mismatch term $\E[\Comp(C)]$. The
\(C^1\)-branch gives an additional explicit algebraic-rate construction. The basic universality statement below uses only the general branch, together with the stronger assumption of exact terminal matching, which removes the terminal mismatch entirely. For the role of the diffusion horizon in the
terminal mismatch error, see Remark~\ref{rem:horizon-terminal-mismatch}.

\begin{theorem}[Conditional-KL denseness under exact terminal matching]
	\label{thm:main}
	Assume the hypotheses of
	Proposition~\ref{prop:main-quant}\ref{item:main-quant-general}.
	If $q_T^{\pi^\star}(\cdot\mid c)=\nu$ for $\mathsf P_{\cC}$-a.e. $c$,
	then there exists a sequence of neural reverse models
	$
	p^{\NN,n}=\bigl\{p_{t,\vartheta_t^{(n)}}^{\NN}\bigr\}_{t=1}^T
	$
	such that
	\begin{equation}
		\lim_{n \to \infty} \mathcal E_{\KL}\!\left(\pi_{p^{\NN,n},\nu}\right)\to 0.
		\label{eq:main-denseness-seq}
	\end{equation}
	Equivalently, the model class \(\mathfrak P_T^{\NN}(\nu)\) is dense at the
	target conditional \(\pi^\star\) in conditional KL. In particular,
	\begin{equation}
		\inf_{\pi\in\mathfrak P_T^{\NN}(\nu)}\mathcal E_{\KL}(\pi)=0.
		\label{eq:main-denseness}
	\end{equation}
\end{theorem}

\begin{proof}
	Under the exact terminal matching hypothesis one has
	$
	\Comp(c)=\KL\!\left(q_T^{\pi^\star}(\cdot\mid c)\,\|\,\nu\right)=0
	$
	for $\mathsf P_{\cC}$-a.e. $c$,
	and hence \(\E[\Comp(C)]=0\).
	Therefore Proposition~\ref{prop:main-quant}\ref{item:main-quant-general}
	implies that for every \(n\ge 1\) there exists a neural reverse model
	\(p^{\NN,n}\) such that
	$
	\mathcal E_{\KL}\!\left(\pi_{p^{\NN,n},\nu}\right)\le \frac{1}{n}
	$.
	This proves \eqref{eq:main-denseness-seq}.  The infimum statement
	\eqref{eq:main-denseness} follows immediately.
\end{proof}

\begin{remark}[Explicit rates and finite-horizon obstruction]
	\label{rem:explicit-rate-under-C1}
	When the $C^1$-branch hypotheses of
	Proposition~\ref{prop:main-quant}\ref{item:main-quant-C1} hold and
	exact terminal matching is in force, the denseness in
	Theorem~\ref{thm:main} is witnessed by models converging at the rate
	$\sum_{t=1}^T C_t^{\mathrm{Nor}}\,n^{-\beta_t}+\frac{1}{2n}$
	from~\eqref{eq:main-finite-n-bound}.  Without exact terminal matching
	the bound plateaus at the irreducible factor $\E[\Comp(C)]$; for the
	Gaussian--OU example of Section~\ref{app:ou-gaussian-example} this
	floor decays exponentially in~$T$ but is strictly positive at every
	finite horizon (see~\eqref{eq:ex-terminal}).
\end{remark}

\begin{corollary}[High-probability and almost-everywhere consequences]
	\label{cor:main-ae}
	Assume the hypotheses of Theorem~\ref{thm:main}.
	\begin{enumerate}[label=(\roman*)]
		\item\label{corItem:etadelta} For every $\eta,\delta>0$, there exists a neural
		reverse model $p^{\NN}$ such that
		\begin{equation}
			\mathsf P_{\cC}\!\left(
			\left\{
			c\in\cC:
			\KL\left(\pi^\star(\cdot\mid c)\,\|\,\pi_{p^{\NN},\nu}(\cdot\mid c)\right)
			>
			\eta
			\right\}
			\right)
			\le
			\delta.
		\end{equation}
		
		\item\label{corItem:to0} There exists a subsequence of the models from
		Theorem~\ref{thm:main}, still denoted $\{p^{\NN,k}\}_{k\ge 1}$, such that
		\begin{equation}
			\KL\left(\pi^\star(\cdot\mid c)\,\|\,\pi_{p^{\NN,k},\nu}(\cdot\mid c)\right)
			\to 0
			\qquad
			\text{for }\mathsf P_{\cC}\text{-a.e. }c.
		\end{equation}
	\end{enumerate}
\end{corollary}

This is proved in Appendix~\ref{sec:proof-cor-main}.

%%%%%%%%%%%%%%%%%%%%%%%%%%%%%%%%%%%%%%%%%%%%%%%%%%%%%%%%%%%%%%%%%%%%%%%%%%%
%--------------------------------------------------------------------------
%%%%%%%%%%%%%%%%%%%%%%%%%%%%%%%%%%%%%%%%%%%%%%%%%%%%%%%%%%%%%%%%%%%%%%%%%%%

\section{Conclusion}
\label{sec:conclusion}

We proved that discrete-time conditional diffusion models with finite
Gaussian-mixture reverse kernels and ReLU-network logits are universal in
conditional KL divergence, up to an explicitly isolated terminal mismatch.
The proof composes the Gaussian-mixture theory of \cite{norets2010} with
the ReLU bounds of \cite{lu2021deep} via a path-space decomposition and
feature-space factorization; Section~\ref{app:ou-gaussian-example}
instantiates the full bound with computable constants for a Gaussian
target with OU forward process.  The result provides a basic expressivity
guarantee for the Gaussian-mixture reverse architecture advocated in
\cite{li2024smd,guo2023gms,gabbur2023improvedddim}, but does not address
optimization, statistical estimation, feature learning, or score
approximation.

Several directions remain open.  On the approximation-theoretic side,
the stepwise rates inherit a response-side curse of dimensionality;
sharpening these rates --- or replacing the ambient target dimension by
a suitable intrinsic one --- is an interesting open problem, since even
at $D=1$ the general theory yields a rate of at most $n^{-1/2}$
(Remark~\ref{rem:ex-rate-optimization}).  On the modeling side, it would
be natural to relax the fixed-component mixture template and to develop
analogous results for continuous-time and score-based formulations
\cite{song2021sde}.  On the statistical side, an important next step is
to combine the expressivity theory with estimation and generalization
guarantees, separating approximation, optimization, and sampling error
in a unified theory of conditional diffusion models.

	\section*{Acknowledgement}
	The work of NI was supported by the Research Start-up Fund of the Shanghai Institute for Mathematics and Interdisciplinary Sciences (SIMIS).

	\appendix

	\section{Technical proofs}\label{sec:proofs}
\subsection{Proof of Proposition~\ref{prop:noretsM0}}\label{sec:proof-prop-noretsM0}
We use three ingredients:
Lemma~\ref{lem:trueReverse}, which gives the density of the true reverse
kernel; the feature-sufficiency hypothesis
Assumption~\ref{ass:feature}\ref{assItem:featureSufficiency}, which lets us write this density as a
function of $z=\phi_t(c,x_t)$; and the pushforward identity
$F_t=((c,x_{t-1},x_t)\mapsto(x_{t-1},\phi_t(c,x_t)))\push\mathsf Q_t$ from \eqref{eq:featureMeas}, which converts integrals against
$\mathsf Q_t$ into integrals against $F_t$. Since Proposition~\ref{prop:noretsM0} makes only $(F_t)_{Z_t}$-a.e.\ and
$F_t$-integrability claims, we fix an arbitrary measurable version of
$f_t(y\mid z)=\rho_t^\star(y\mid z)$ on $\R^D\times\R^{d_t}$, still denoted
$f_t$, agreeing with $\rho_t^\star$ on the $(F_t)_{Z_t}$-full set where the
latter is defined.

\paragraph{Bridge identity on feature space.}
Combining Lemma~\ref{lem:trueReverse} and Assumption~\ref{ass:feature}\ref{assItem:featureSufficiency}, we obtain: for
$(\mathsf Q_t)_{C,X_t}$-a.e.\ $(c,x_t)$ and Leb-a.e.\ $y$,
\begin{equation}
	f_t\left(y\mid \phi_t(c,x_t)\right)
	=
	\frac{q_{t-1}(x_t\mid y)\,
		q_{t-1}^{\pi^\star}(y\mid c)}
	{q_t^{\pi^\star}(x_t\mid c)}.
	\label{eq:bridge-on-features}
\end{equation}
We shall refer to \eqref{eq:bridge-on-features} as the
\emph{bridge identity on feature space}.  The right-hand side is a
concrete expression in terms of the forward densities.

\paragraph{Continuity and strict positivity of forward marginals.}
We first record, by induction on $s\in\{0,\dots,T\}$, that for
$\mathsf P_{\cC}$-a.e.\ $c$ the forward marginal density
$u\mapsto q_s^{\pi^\star}(u\mid c)$ is continuous and strictly positive
on~$\cX$.  The base case $s=0$ is exactly
Assumption~\ref{ass:regularity}\ref{assItem:positivity}--\ref{assItem:continuity}
applied to $\pi^\star$.  For the inductive step, write
\begin{equation}
	q_s^{\pi^\star}(u\mid c)
	=\int_\cX q_{s-1}(u\mid v)\,q_{s-1}^{\pi^\star}(\dd v\mid c).
\end{equation}
Continuity in $u$ follows from dominated convergence using
Assumption~\ref{ass:regularity}\ref{assItem:continuity}.  Strict
positivity follows from the positivity of $q_{s-1}(u\mid v)$ and
$q_{s-1}^{\pi^\star}(v\mid c)$.

\medskip

We now verify each of the three target-side conditions.

\paragraph{Item~\ref{item:noretsM0-pos}: Positivity and continuity in
	$y$.}
Consider the right-hand side of \eqref{eq:bridge-on-features}.  For
fixed $(c,x_t)$ in the full-measure set where \eqref{eq:bridge-on-features}
holds, the denominator $q_t^{\pi^\star}(x_t\mid c)$ is a strictly
positive constant in $y$, while the numerator is a product of the
functions $y\mapsto q_{t-1}(x_t\mid y)$ and
$y\mapsto q_{t-1}^{\pi^\star}(y\mid c)$, both of which are continuous
and strictly positive by the induction above and
Assumption~\ref{ass:regularity}\ref{assItem:positivity}--\ref{assItem:continuity}.
Hence, for $(\mathsf Q_t)_{C,X_t}$-a.e.\ $(c,x_t)$, the function
$y\mapsto f_t(y\mid\phi_t(c,x_t))$ agrees Lebesgue-a.e.\ with
a continuous, strictly positive function of $y$. If $\phi_t$ is many-to-one, the right-hand side of
\eqref{eq:bridge-on-features} may a priori depend on the representative
$(c,x_t)$ with $\phi_t(c,x_t)=z$. For $(F_t)_{Z_t}$-a.e.\ $z$, any two such
representatives for which \eqref{eq:bridge-on-features} holds yield functions
that agree with $f_t(\cdot\mid z)$ for $\Leb$-a.e.\ $y$; by
continuity in $y$, they therefore coincide everywhere.

Since $(F_t)_{Z_t}$ is the image of $(\mathsf Q_t)_{C,X_t}$ under
$(c,x_t)\mapsto\phi_t(c,x_t)$, the set of feature values $z$ for which such a
``good'' $(c,x_t)$ exists has full $(F_t)_{Z_t}$-measure.  Because $f_t(\cdot\mid z)$ is defined $(F_t)_{Z_t}$-a.e.\ and agrees
$\Leb$-a.e.\ with a continuous function, we may redefine $f_t(\cdot\mid z)$
on the Lebesgue-null discrepancy set to obtain a version that is everywhere
continuous and strictly positive for $(F_t)_{Z_t}$-a.e.\ $z$.

\paragraph{Item~\ref{item:noretsM0-M2}: Finite second moment.}
By definition of $F_t$ \eqref{eq:featureMeas} and $\mathsf Q_t$ (Def.~\ref{def:forwardJoint}),
\begin{equation}
	\int \|y\|^2\,F_t(\dd y,\dd z)
	=
	\int_{\cC\times\R^D\times\R^D} \|y\|^2\,\mathsf Q_t(\dd c,\dd y,\dd x_t)
	=
	\int_{\cC\times\R^D} \|y\|^2\,q_{t-1}^{\pi^\star}(\dd y\mid c)\,
	\mathsf P_\cC(\dd c),
\end{equation}
which is finite by
Assumption~\ref{ass:regularity}\ref{assItem:finiteM2}.

\paragraph{Item~\ref{item:noretsM0-log}: Integrable local
	log-variation.}
Fix $t\ge 1$.  We work on $\mathsf Q_t$ and transfer to $F_t$ at the end.

\emph{Pointwise bound.}
From \eqref{eq:bridge-on-features}, for
$(\mathsf Q_t)_{C,X_t}$-a.e.\ $(c,x_t)$ and every $y\in\R^D$,
\begin{equation}
	\begin{aligned}
		\log\frac{f_t(y\mid\phi_t(c,x_t))}
		{\inf_{u\in C_{t-1}(y)}
			f_t(u\mid\phi_t(c,x_t))}
		\;\le\;
		\log\frac{q_{t-1}(x_t\mid y)}
		{\inf_{u\in C_{t-1}(y)}q_{t-1}(x_t\mid u)}
		+
		\log\frac{q_{t-1}^{\pi^\star}(y\mid c)}
		{\inf_{u\in C_{t-1}(y)}
			q_{t-1}^{\pi^\star}(u\mid c)}.
		\label{eq:ft-log-variation-split}
	\end{aligned}
\end{equation}
The inequality $\inf(AB)\ge(\inf A)(\inf B)$ for nonnegative functions
is used to split the ratio.

The first term on the right is bounded by $b_{t-1}(y,x_t)$ by
\eqref{eq:conditioning-variable-oscillation} and the second term by $a_{t-2}(y)$ by
\eqref{eq:marginal-oscillation} with $s=t-1$.
So the combined bound is:
for $(\mathsf Q_t)_{C,X_t}$-a.e.\ $(c,x_t)$ and every $y$,
\begin{equation}
	\log\frac{f_t(y\mid\phi_t(c,x_t))}
	{\inf_{u\in C_{t-1}(y)}
		f_t(u\mid\phi_t(c,x_t))}
	\;\le\;
	b_{t-1}(y,x_t)+a_{t-2}(y).
	\label{eq:pointwise-log-variation}
\end{equation}

\emph{Integration and transfer to $F_t$.}
Let
\begin{equation}
	h(y,z):=
	\log\frac{f_t(y\mid z)}
	{\inf_{u\in C_{t-1}(y)}f_t(u\mid z)}.
\end{equation}
Since $h$ depends on $(c,x_t)$ only through $z=\phi_t(c,x_t)$, the
pushforward identity $F_t=((c,x_{t-1},x_t)\mapsto(x_{t-1},\phi_t(c,x_t)))\push\mathsf Q_t$ gives
\begin{equation}
	\int h(y,z)\,F_t(\dd y,\dd z)
	\;=\;
	\int_{\cC\times\R^D\times\R^D}
	h\left(x_{t-1}, \phi_t(c,x_t)\right)\,
	\mathsf Q_t(\dd c,\dd x_{t-1},\dd x_t).
	\label{eq:Ft-Qt-transfer}
\end{equation}
By \eqref{eq:pointwise-log-variation}, the integrand on the right
satisfies
\begin{equation}
	h\left(x_{t-1},\phi_t(c,x_t)\right)
	\le
	b_{t-1}(x_{t-1},x_t)+ a_{t-2}(x_{t-1})
\end{equation}
for $\mathsf Q_t$-a.e.\ $(c,x_{t-1},x_t)$.  Hence the right-hand side
of \eqref{eq:Ft-Qt-transfer} is bounded by
\begin{equation}
	\E_{(C,X_{t-1},X_t)\sim\mathsf Q_t}
	\left[b_{t-1}(X_{t-1},X_t)
	+  a_{t-2}(X_{t-1})\right],
\end{equation}
which is finite by Assumption~\ref{ass:regularity}\ref{assItem:logBound}.  This proves
item~\ref{item:noretsM0-log}.

\subsection{Proof of Proposition~\ref{prop:main-quant}}\label{sec:proof-main-quant}
\paragraph{Proof of~\ref{item:main-quant-general}.}
Fix $\varepsilon>0$.
For each $t\in\{1,\dots,T\}$, by
Proposition~\ref{prop:M0-explicit-bound} we have
$d_{\KL,t}(\tilde p_t^{\mathcal M_0,m_t})
\le B_{t,m_t}^{\mathrm{Nor}}$
for all sufficiently large~$m_t$.
Each constituent term in the bound $B_{t,m_t}^{\mathrm{Nor}}$
\eqref{eq:Bnorets} vanishes as $m_t\to\infty$ under the standing
assumptions: $L_t(\delta_{t,m_t})\to 0$ by continuity of $f_t(\cdot\mid z)$
and dominated convergence;
$T_{t,m_t}(\delta_{t,m_t})\to 0$ and
$U_{t,m_t}(\delta_{t,m_t})\to 0$ because the boundary event
$\mathcal B_{t,m_t}(\delta_{t,m_t})$ shrinks to a null set under $F_t$;
the fourth term tends to zero by \eqref{eq:delta-sigma-h-condition};
and the fifth by $\sigma_{t,m_t}/\delta_{t,m_t}\to 0$.
Hence
\begin{equation}
	d_{\KL,t}(\tilde p_t^{\mathcal M_0,m_t})\to 0
	\qquad\text{as }m_t\to\infty.
	\label{eq:M0-vanishing}
\end{equation}
Choose $m_t$ large enough that
\begin{equation}
	d_{\KL,t}(\tilde p_t^{\mathcal M_0,m_t})
	<\frac{\varepsilon}{2T}.
	\label{eq:mt-choice}
\end{equation}
For this fixed $m_t$, the log-odds complexity
$\Lambda_{t,m_t}^{\max}$ is finite by
Assumption~\ref{ass:smooth-logodds}, so
Proposition~\ref{prop:nn-logit-rate} provides a ReLU network
$a_t^{\vartheta_t}$ with logit error $\eta_t^{\NN}$ that can be made
arbitrarily small. Choose this network so that
\begin{equation}
	\eta_t^{\NN}<\frac{\varepsilon}{4T}.
	\label{eq:eta-choice}
\end{equation}
By Lemma~\ref{lem:baseline-softmax-stability},
\begin{equation}
	d_{\KL,t}(\tilde p_{t,\vartheta_t}^{\NN})
	\le
	d_{\KL,t}(\tilde p_t^{\mathcal M_0,m_t})+2\eta_t^{\NN}
	<\frac{\varepsilon}{2T}+\frac{\varepsilon}{2T}
	=\frac{\varepsilon}{T}.
\end{equation}
Summing over $t=1,\dots,T$ and applying
Lemma~\ref{lem:path-kl} and Lemma~\ref{lem:feature-factorization} gives
\begin{equation}
	\mathcal E_{\KL}\left(\pi_{p^{\NN},\nu}\right)
	\le
	\E[\Comp(C)]
	+\sum_{t=1}^T d_{\KL,t}(\tilde p_{t,\vartheta_t}^{\NN})
	<\E[\Comp(C)]+\varepsilon.
\end{equation}
Since $\varepsilon>0$ was arbitrary, \eqref{eq:main-general-inf} follows.

\medskip

\paragraph{Proof of~\ref{item:main-quant-C1}.}
Fix $t\in\{1,\dots,T\}$. By Corollary~\ref{cor:M0-power-rate}, there exists an
integer $n_{0,t}\ge 1$ such that
\begin{equation}
	d_{\KL,t}\left(\tilde p_t^{\mathcal M_0,m}\right)
	\le
	C_t^{\mathrm{Nor}}\,m^{-\beta_t}
	\qquad
	\forall m\ge n_{0,t}.
\end{equation}
Let $n_0:=\max_{1\le t\le T} n_{0,t}$. Now fix $n\ge n_0$, and set
\begin{equation}
	m_t:=n
	\qquad
	\text{for every }t=1,\dots,T.
\end{equation}

For this fixed $n$ and fixed $t$, the log-odds complexity
$\Lambda_{t,n}^{\max}$ defined in Assumption~\ref{ass:smooth-logodds} is finite by
Assumption~\ref{ass:smooth-logodds}. By
Proposition~\ref{prop:nn-logit-rate}, for any integers $\mathscr N_t,\mathscr L_t\ge 1$ there
exists a ReLU network $a_t^{\vartheta_t}$ with logit error
\begin{equation}
	\eta_t^{\NN}
	\le
	85\,(s_t+1)^{d_t}8^{s_t}\,
	\Lambda_{t,n}^{\max}\,
	\mathscr N_t^{-2s_t/d_t}\mathscr L_t^{-2s_t/d_t}.
\end{equation}
Since $\Lambda_{t,n}^{\max}<\infty$ and the right-hand side tends to zero as
$\mathscr N_t,\mathscr L_t\to\infty$, we may choose integers $\mathscr N_t^{(n)},\mathscr L_t^{(n)}\ge 1$ and a
corresponding ReLU network $a_t^{\vartheta_t^{(n)}}$ such that
\begin{equation}
	\eta_t^{\NN}
	=
	\max_{1\le j\le n}\sup_{z\in K_t}
	\left|
	a_{t,j}^{\vartheta_t^{(n)}}(z)-\Lambda_{t,j}^{n}(z)
	\right|
	\le
	\frac{1}{4Tn}.
	\label{eq:main-eta-choice}
\end{equation}
Let $p_{t,\vartheta_t^{(n)}}^{\NN}$ be the resulting step-$t$ neural reverse
kernel, and define
$
p^{\NN,n}
=
\Bigl\{
p_{t,\vartheta_t^{(n)}}^{\NN}
\Bigr\}_{t=1}^T
$.

Applying Corollary~\ref{cor:M0-power-rate} and Lemma~\ref{lem:baseline-softmax-stability} at each step, we get
\begin{equation}
	d_{\KL,t}\left(\tilde p_{t,\vartheta_t^{(n)}}^{\NN}\right)
	\le
	C_t^{\mathrm{Nor}}\,n^{-\beta_t}
	+
	2\eta_t^{\NN}
	\qquad
	\forall t=1,\dots,T.
\end{equation}
Therefore, by Lemma~\ref{lem:path-kl},
\begin{align}
	\mathcal E_{\KL}\left(\pi_{p^{\NN,n},\nu}\right)
	&\le
	\E[\Comp(C)]
	+
	\sum_{t=1}^T
	d_{\KL,t}\left(\tilde p_{t,\vartheta_t^{(n)}}^{\NN}\right)
	\nn\\
	&\le
	\E[\Comp(C)]
	+
	\sum_{t=1}^T
	\left(
	C_t^{\mathrm{Nor}}\,n^{-\beta_t}
	+
	2\eta_t^{\NN}
	\right)
	\nn\\
	&\le
	\E[\Comp(C)]
	+
	\sum_{t=1}^T
	\left(
	C_t^{\mathrm{Nor}}\,n^{-\beta_t}
	+
	\frac{1}{2Tn}
	\right) \qquad \text{(using \eqref{eq:main-eta-choice})}
	\nn\\
	&=
	\E[\Comp(C)]
	+
	\sum_{t=1}^T C_t^{\mathrm{Nor}}\,n^{-\beta_t}
	+
	\frac{1}{2n}.
	\label{eq:main-proof-chain}
\end{align}
This proves \eqref{eq:main-finite-n-bound}.
Taking $\limsup_{n\to\infty}$ gives \eqref{eq:main-limsup-bound}.

\subsection{Proof of Corollary~\ref{cor:main-ae}}\label{sec:proof-cor-main}
For each neural reverse model $p^{\NN}$, define the contextwise output-KL
function
\begin{equation}
	K_{p^{\NN}}(c)
	:=
	\KL\left(\pi^\star(\cdot\mid c)\,\|\,\pi_{p^{\NN},\nu}(\cdot\mid c)\right),
	\qquad
	c\in\cC.
\end{equation}
\begin{equation}
	\text{Then,} \qquad \E_{C\sim\mathsf P_{\cC}} \left[K_{p^{\NN}}(C)\right]
	=
	\mathcal E_{\KL}\left(\pi_{p^{\NN},\nu}\right).
\end{equation}

\paragraph{Proof of \ref{corItem:etadelta}.} Fix $\eta,\delta>0$. By Theorem~\ref{thm:main}, there exists a neural
reverse model $p^{\NN}$ such that
\begin{equation}
	\mathcal E_{\KL}\left(\pi_{p^{\NN},\nu}\right)\le \eta\delta.
\end{equation}
By Markov's inequality,
\begin{equation}
	\begin{aligned}
		\mathsf P_{\cC}\left(K_{p^{\NN}}(C)>\eta\right)
		\le
		\frac{\E[K_{p^{\NN}}(C)]}{\eta}
		\le
		\delta.
	\end{aligned}
\end{equation}

\paragraph{Proof of \ref{corItem:to0}.}
By Theorem~\ref{thm:main}, there exists a sequence of neural reverse models
$\{p^{\NN,n}\}_{n\ge1}$ such that
$
\mathcal E_{\KL}\!\left(\pi_{p^{\NN,n},\nu}\right)\to 0
$.
Hence we may choose a subsequence $\{p^{\NN,n_k}\}_{k\ge1}$ such that
\begin{equation}
\mathcal E_{\KL}\!\left(\pi_{p^{\NN,n_k},\nu}\right)
=
\E[K_{p^{\NN,n_k}}(C)]
\le 2^{-2k}
\qquad
\forall k\ge1.
\end{equation}
Markov's inequality therefore gives
\begin{equation}
\mathsf P_{\cC}\!\left(K_{p^{\NN,n_k}}(C)>2^{-k}\right)\le 2^{-k}.
\end{equation}
Since $\sum_{k=1}^\infty 2^{-k}<\infty$, the Borel--Cantelli lemma implies that,
for $\mathsf P_{\cC}$-a.e.\ $c$, the event
$
K_{p^{\NN,n_k}}(c)>2^{-k}
$
occurs only finitely often. Equivalently, for $\mathsf P_{\cC}$-a.e.\ $c$,
$
K_{p^{\NN,n_k}}(c)\to 0
$.
Renaming the subsequence proves the claim.

%%%%%%%%%%%%%%%%%%%%%%%%%%%%%%%%%%%%%%%%%%%%%%%%%%%%%%%%%%%%%%%%%%%%%%%%%%%
%--------------------------------------------------------------------------
%%%%%%%%%%%%%%%%%%%%%%%%%%%%%%%%%%%%%%%%%%%%%%%%%%%%%%%%%%%%%%%%%%%%%%%%%%%

\section{Example: Gaussian Target under Ornstein--Uhlenbeck Noising}
\label{app:ou-gaussian-example}

This section constructs a fully explicit one-dimensional example in which every
quantity appearing in the $C^1$-branch bound
(Proposition~\ref{prop:main-quant}\ref{item:main-quant-C1}) can be evaluated
with concrete leading coefficients.  The example is deliberately simple: the
target conditional is Gaussian, the forward process is Ornstein--Uhlenbeck (OU), and
the response and feature dimensions are both equal to one.  Despite its simplicity, this example is rich enough to illustrate all the principal components of the approximation theory developed in the main text. At the same time, it exhibits a practically relevant situation in which the feature distribution has decaying non-compact support, thereby showing how the technical issues set aside in the main exposition may be treated in practice.

\subsection{Setup and closed-form quantities}
\label{sec:ex-setup}

\subsubsection{Problem data}

Let $\cC=\cX=\R$ (so $D=1$), and let the context have the uniform distribution
\begin{equation}
	C\sim\mathrm{Unif}[-M,M]
\end{equation}
for a fixed $M>0$.
The true conditional target is Gaussian with $c$-dependent mean and unit
variance:
\begin{equation}
	\pi^\star(x_0\mid c)=\varphi(x_0;\,c,\,1),
\end{equation}
where $\varphi(\,\cdot\,;\mu,\sigma^2)$ denotes the
$\mathcal N(\mu,\sigma^2)$ density throughout.

Fix a contraction parameter $\rho\in(0,1)$ and write
\begin{equation}
	\sigma^2:=1-\rho^2
\end{equation}
for the OU noise variance.  The forward kernels are time-homogeneous
OU steps:
\begin{equation}\label{eq:ex-ou}
	q_t(x_{t+1}\mid x_t)
	=\varphi(x_{t+1};\,\rho\, x_t,\,\sigma^2),
	\qquad t=0,\dots,T-1.
\end{equation}
Intuitively, each forward step contracts the signal by a factor of $\rho$ and
adds independent Gaussian noise of variance $\sigma^2$.
The reference (terminal) law is
\begin{equation}
	\nu=\mathcal N(0,1),
\end{equation}
which is the stationary distribution of the OU chain.

\begin{remark}[Choice of unit initial variance]\label{rem:ex-general-v}
	Taking $\pi^\star(\cdot\mid c)=\mathcal N(c,1)$ ensures
	$V_t:=\mathrm{Var}(X_t\mid C=c)\equiv 1$ for all $t$, which makes
	the reverse-kernel variance constant across steps and considerably
	simplifies the algebra.  For general initial variance $v^2>0$ one has
	$V_t=1+\rho^{2t}(v^2-1)$; all formulas below carry through with
	$\widetilde V_t=\sigma^2 V_{t-1}/V_t$ and $t$-dependent coefficients.  The only
	effect on the final bound is an additional variance-mismatch contribution
	$\tfrac14\rho^{4T}(v^2-1)^2+O(\rho^{6T})$ in the terminal mismatch.
\end{remark}

\subsubsection{Forward marginals}

We claim that for every $t\in\{0,\dots,T\}$,
\begin{equation}\label{eq:ex-marginal}
	X_t\mid C=c\;\sim\;\mathcal N(\rho^t c,\,1).
\end{equation}
This follows by induction.  The base case $t=0$ is just the definition of
$\pi^\star$.  For the inductive step, suppose
$X_t\mid C=c\sim\mathcal N(\rho^t c,1)$.
Since $X_{t+1}=\rho X_t+\sigma\varepsilon$ with $\varepsilon\sim\mathcal N(0,1)$
independent of $X_t$, we get
\begin{equation}
	X_{t+1}\mid C=c
	\;\sim\;
	\mathcal N\left(\rho\cdot\rho^t c,\;\rho^2\cdot 1+\sigma^2\right)
	=\mathcal N(\rho^{t+1}c,\,1),
\end{equation}
using $\rho^2+\sigma^2=1$.

In particular, the forward marginals are all unit-variance Gaussians whose
means decay geometrically toward zero.  At the terminal time $T$, the marginal
$q_T^{\pi^\star}(\cdot\mid c)=\mathcal N(\rho^T c,1)$ is close to --- but not
exactly equal to --- the reference $\nu=\mathcal N(0,1)$.

\subsubsection{Reverse kernels and the sufficient feature}

We now compute the true reverse kernel $r_t^\star(y\mid c,x_t)$.
By the bridge identity (Lemma~\ref{lem:trueReverse}), this density is
proportional in $y$ to the product
\begin{equation}
	q_{t-1}(x_t\mid y)  q_{t-1}^{\pi^\star}(y\mid c)
	=\varphi(x_t;\rho y,\sigma^2) \varphi(y;\rho^{t-1}c,1).
\end{equation}
Both factors are Gaussian in $y$.  Completing the square
(equivalently, applying the standard product-of-Gaussians formula) gives
\begin{equation}\label{eq:ex-reverse}
	r_t^\star(y\mid c,x_t)
	=\varphi(y;\,\underbrace{\rho\, x_t+\sigma^2\rho^{t-1}c}_{=:\,\phi_t(c,x_t)},
	\;\sigma^2),
	\qquad t=1,\dots,T.
\end{equation}
Here is the key structural observation: the reverse kernel depends on the pair
$(c,x_t)$ only through the single scalar combination
\begin{equation}\label{eq:ex-feature}
	z=\phi_t(c,x_t):=\rho\, x_t+\sigma^2\rho^{t-1}c.
\end{equation}
This is the \emph{sufficient feature}, and its dimension is
\begin{equation}
	d_t=1.
\end{equation}
The reverse density, viewed as a function of the feature, is simply a Gaussian
with mean $z$ and variance $\sigma^2$:
\begin{equation}\label{eq:ex-ft}
	f_t(y\mid z):=\rho_t^\star(y\mid z)=\varphi(y;\,z,\,\sigma^2).
\end{equation}
Thus Assumption~\ref{ass:feature}\ref{assItem:featureSufficiency} (feature
sufficiency) holds exactly, without any modification of the feature map.

The joint law $F_t$ of $(Y,Z):=(X_{t-1},Z_t)$ (Definition~\ref{def:forwardJoint}
and~\eqref{eq:featureMeas}) is obtained by marginalizing over $C$.
Conditionally on $C=c$, the pair is jointly Gaussian:
\begin{equation}\label{eq:ex-Ft}
	\begin{pmatrix} Y \\ Z \end{pmatrix}
	\;\bigg|\; C=c
	\;\sim\;
	\mathcal N \left(
	\begin{pmatrix} \rho^{t-1}c \\ \rho^{t-1}c \end{pmatrix},\;
	\begin{pmatrix} 1 & \rho^2 \\ \rho^2 & \rho^2 \end{pmatrix}
	\right),
\end{equation}
so $Y\mid Z=z\sim\mathcal N(z,\sigma^2)$ and
$Y-Z\mid C=c\sim\mathcal N(0,\sigma^2)$.  In particular,
\begin{equation}\label{eq:ex-Ft-moments}
	\E|Y-Z|=\sigma\sqrt{2/\pi},
	\qquad
	\E[(Y-Z)^2]=\sigma^2,
	\qquad
	\E[Y^2]=\frac{1}{3}\rho^{2(t-1)}M^2+1.
\end{equation}
%These moments will be cited repeatedly in the error-term computations below.

\subsubsection{Terminal mismatch}

Since $q_T^{\pi^\star}(\cdot\mid c)=\mathcal N(\rho^T c,1)$ and
$\nu=\mathcal N(0,1)$, the KL divergence between two Gaussians with equal
variance gives
\begin{equation}
	\Comp(c)
	=\KL\left(\mathcal N(\rho^T c,1)\,\|\,\mathcal N(0,1)\right)
	=\frac{(\rho^T c)^2}{2}
	=\frac12\rho^{2T}c^2.
\end{equation}
Averaging over $C\sim\mathrm{Unif}[-M,M]$, using $\E[C^2]=M^2/3$:
\begin{equation}\label{eq:ex-terminal}
	\E[\Comp(C)]=\frac{\rho^{2T}M^2}{6}.
\end{equation}
This is exponentially small in $T$ with rate $\rho^{2T}$.

Exact terminal matching $q_T^{\pi^\star}(\cdot\mid c)=\nu$ for
$\mathsf P_\cC$-a.e.\ $c$ would require $\rho^T c=0$ for all
$c\in[-M,M]$, which is impossible for $M>0$ and finite $T$.  Hence the
terminal mismatch \eqref{eq:ex-terminal} is genuine and irreducible at finite
horizon~$T$; it vanishes only asymptotically.

\subsection{Assumption verification}
\label{sec:ex-assumptions}

\subsubsection{Regularity of forward densities
	(Assumption~\texorpdfstring{\ref{ass:regularity}}{3.4})}
\label{sec:verifyForwardAss}

We verify all four parts of Assumption~\ref{ass:regularity}.  Throughout,
we set $r_s:=1$ and take the hypercubes to be the intervals
$C_s(x):=[x-\tfrac12,x+\tfrac12]$ for every $s=0,\dots,T-1$. % Write $B_t:=\rho^{t-1}M$ for the maximum absolute conditional mean of the feature $Z_t$.

\paragraph{\ref{assItem:positivity}~Strict positivity.}
The target density
$\pi^\star(x_0\mid c)=\varphi(x_0;c,1)$ is strictly positive for all
$x_0\in\R$ and all $c$.  The forward transition density
$q_s(x_{s+1}\mid x_s)=\varphi(x_{s+1};\rho x_s,\sigma^2)$ is strictly positive
for all $s,x_s,x_{s+1}$.

\paragraph{\ref{assItem:continuity}~Continuity and local integrable
	dominance.}
All densities involved are Gaussian, hence $C^\infty$ in all arguments.
For any compact $K\subset\R$,
$\sup_{u\in K}q_s(u\mid v)\le(2\pi \sigma^2)^{-1/2}$, which is a finite
constant independent of $v$ and hence trivially integrable against any
probability measure. 

\paragraph{\ref{assItem:finiteM2}~Finite second moments.}
By \eqref{eq:ex-marginal}, $X_t\mid C=c\sim\mathcal N(\rho^t c,1)$,
so $\E[X_t^2]=\rho^{2t}\E[C^2]+1=\rho^{2t}M^2/3+1<\infty$ for every
$t=0,\dots,T-1$. 

\paragraph{\ref{assItem:logBound}~Integrable local log-variation bounds.}
This is the most involved verification.  We need to produce the bounding
functions $a_{s-1}$ and $b_s$ and check their integrability.

\emph{Marginal oscillation bound \eqref{eq:marginal-oscillation}.}
Recall that the forward marginal at time $s$ is
$q_s^{\pi^\star}(x_s\mid c)=\varphi(x_s;\rho^s c,1)$.  Since the Gaussian
density $\varphi(\,\cdot\,;\mu,1)$ is log-concave and radially decreasing
around its mean $\mu$, the infimum over an interval
$C_s(x_s)=[x_s-\frac12,x_s+\frac12]$ is attained at the endpoint farthest
from $\mu=\rho^s c$.  Therefore
\begin{equation}
	\log\frac{q_s^{\pi^\star}(x_s\mid c)}
	{\inf_{y\in C_s(x_s)}q_s^{\pi^\star}(y\mid c)}
	=\frac{(|x_s-\rho^s c|+\frac12)^2-(x_s-\rho^s c)^2}{2}
	=\frac{|x_s-\rho^s c|+\tfrac14}{2}.
\end{equation}
Since $|\rho^s c|\le\rho^s M$ for $c\in[-M,M]$, this is bounded by
\begin{equation}\label{eq:ex-as-def}
	a_{s-1}(x_s)
	:=\frac{|x_s|+\rho^s M+\tfrac14}{2},
\end{equation}
which depends only on $x_s$ (not on $c$).

\emph{Conditioning-variable oscillation bound
	\eqref{eq:conditioning-variable-oscillation}.}
The forward transition is
$q_s(x_{s+1}\mid x_s)=\varphi(x_{s+1};\rho x_s,\sigma^2)$.  By the same
farthest-point argument (now the infimum is over $x_s$ varying in $C_s(x_s)$,
with the Gaussian centered at $\rho^{-1}(x_{s+1}-\cdot)$):
\begin{equation}
	\log\frac{q_s(x_{s+1}\mid x_s)}
	{\inf_{y\in C_s(x_s)}q_s(x_{s+1}\mid y)}
	%=\frac{\frac{\rho}{2}(2|x_{s+1}-\rho x_s|+\frac{\rho}{2})}{2\sigma^2}
	=\frac{\rho\,|x_{s+1}-\rho\, x_s|+\frac{\rho^2}{4}}{2\sigma^2}.
\end{equation}
We set
\begin{equation}\label{eq:ex-bs-def}
	b_s(x_s,x_{s+1}):=\frac{\rho\,|x_{s+1}-\rho\, x_s|+\frac{\rho^2}{4}}{2\sigma^2}.
\end{equation}

\emph{Integrability under $\mathsf Q_{s+1}$.}
Under $\mathsf Q_{s+1}(\dd c, \dd x_s, \dd x_{s+1})= \allowbreak \mathrm{Unif}[-M,M](\dd c) \mathcal N(\rho^sc,1)(\dd x_s) \allowbreak \mathcal N(\rho x_s,\sigma^2)(\dd x_{s+1})$, the random variable $X_s$ has finite first moment
(since $\E|X_s|\le(\E[X_s^2])^{1/2}=(\rho^{2s}M^2/3+1)^{1/2}$), and
$X_{s+1}-\rho X_s$ is conditionally $\mathcal N(0,\sigma^2)$ given $X_s$, so
$\E|X_{s+1}-\rho X_s|=\sigma\sqrt{2/\pi}$.  Therefore
\begin{align}
	\E_{\mathsf Q_{s+1}}[a_{s-1}(X_s)]
	&=\frac{\E|X_s|+\rho^s M+\tfrac14}{2}<\infty,
	\label{eq:ex-as-integrable}\\
	\E_{\mathsf Q_{s+1}}[b_s(X_s,X_{s+1})]
	&=\frac{\rho\,\E|X_{s+1}-\rho X_s|+\frac{\rho^2}{4}}{2\sigma^2}
	=\frac{\rho}{2\sigma}\sqrt{\frac{2}{\pi}}+\frac{\rho^2}{8\sigma^2}<\infty.
	\label{eq:ex-bs-integrable}
\end{align}
%In particular, for any concrete parameter values one can evaluate these moments numerically.  For example, with $\rho=0.9$ (so $\sigma\approx 0.436$) and $M=2$, the bound \eqref{eq:ex-bs-integrable} evaluates to approximately $0.9\cdot 0.798/(2\cdot 0.436)+0.81/(8\cdot 0.19)\approx 1.36$. 

\subsubsection{Norets $\mathcal M_0$ conditions
	(Proposition~\texorpdfstring{\ref{prop:noretsM0}}{3.5})}\label{sec:ex-M0-assumptions}

All parts of Assumptions \ref{ass:feature} and \ref{ass:regularity} are satisfied except \ref{ass:feature}\ref{assItem:compactFeature} which incurs additional tail contribution. We show by direct computation that Proposition~\ref{prop:noretsM0} still holds.

\emph{Item~\ref{item:noretsM0-pos}: Positivity and continuity.}
The function $y\mapsto\varphi(y;\,z,\,\sigma^2)$ is $C^\infty$ and strictly
positive for every $z\in\R$. 

\emph{Item~\ref{item:noretsM0-M2}: Finite second moment.}
The relevant integral is (from \eqref{eq:ex-Ft-moments})
$\int|y|^2\,F_t(\dd y,\dd z)=\E[X_{t-1}^2]=\rho^{2(t-1)}M^2/3+1<\infty$.

\emph{Item~\ref{item:noretsM0-log}: Integrable local log-variation.}
With $r_{t-1}=1$, the log-variation of $f_t(y\mid z)=\varphi(y;z,\sigma^2)$ over
$C_1(y)=[y-\frac12,y+\frac12]$ is computed as follows.  Since
$\log f_t(y\mid z)=-\frac12\log(2\pi \sigma^2)-(y-z)^2/(2\sigma^2)$, the infimum
over $u\in C_1(y)$ is attained at the point farthest from $z$, giving
\begin{equation}\label{eq:ex-log-variation}
	\log\frac{f_t(y\mid z)}
	{\inf_{u\in C_1(y)}f_t(u\mid z)}
	=\frac{(|y-z|+\frac12)^2-(y-z)^2}{2\sigma^2}
	=\frac{|y-z|+\tfrac14}{2\sigma^2}.
\end{equation}
Integrating this against $F_t$ (using \eqref{eq:ex-Ft-moments}) we get
\begin{equation}
	\int\frac{|y-z|+\tfrac14}{2\sigma^2}\,F_t(\dd y,\dd z)
	=\frac{\E|Y-Z|}{2\sigma^2}+\frac{1}{8\sigma^2}
	=\frac{1}{2\sigma}\sqrt{\frac{2}{\pi}}+\frac{1}{8\sigma^2}<\infty.
\end{equation}

\subsubsection{Feature law and growing compact windows}\label{sec:windows}

Recall from \eqref{eq:ex-feature} that the sufficient feature is
$
Z_t=\phi_t(C,X_t)=\rho\,X_t+\sigma^2\rho^{t-1}C
$.
Since \(X_t\mid C=c\sim\mathcal N(\rho^t c,1)\) by \eqref{eq:ex-marginal},
the conditional law of \(Z_t\) is again Gaussian:
\begin{equation}
	Z_t\mid C=c
	\sim
	\mathcal N(\rho^{t-1}c,\rho^2).
\end{equation}
Indeed,
$
\E[Z_t\mid C=c]
\allowbreak=
\rho\,\E[X_t\mid C=c]+\sigma^2\rho^{t-1}c
\allowbreak=
\rho^{t+1}c+\sigma^2\rho^{t-1}c
\allowbreak=
\rho^{t-1}c
$,
using \(\rho^2+\sigma^2=1\), and
$
\mathrm{Var}(Z_t\mid C=c)
\allowbreak=
\rho^2\,\mathrm{Var}(X_t\mid C=c)
\allowbreak=
\rho^2
$.

Thus the unconditional law of \(Z_t\) is a bounded location-mixture of Gaussians,
hence not compactly supported.  Rather than altering the feature map, we work
on growing compact windows and keep track of the tail error explicitly.  For
each \(n\ge1\), define
\begin{equation}\label{eq:ex-Ktn}
	K_{t,n}:=[-R_n,R_n],
	\qquad
	R_n:=\tfrac14 n^{1/7}.
\end{equation}

In what follows we shall use the basic Gaussian tail estimate a few times with certain parameter relations.
\begin{remark}[Gaussian tail estimate]
Let \(W\sim\mathcal N(\mu,s^2)\) with \(|\mu|\le m\).  Then for every \(r\ge m\),
\begin{equation}\label{eq:ex-gaussian-tail-template}
	\Prob(|W|>r)
	\le
	2\,\bar\Phi\!\left(\frac{r-m}{s}\right).
\end{equation}
Indeed,
\begin{equation}
\Prob(|W|>r)
=
\bar\Phi\!\left(\frac{r-\mu}{s}\right)
+
\bar\Phi\!\left(\frac{r+\mu}{s}\right),
\end{equation}
and since \(|\mu|\le m\) and \(\bar\Phi\) is decreasing,
$
\bar\Phi\!\left(\frac{r-\mu}{s}\right)
+
\bar\Phi\!\left(\frac{r+\mu}{s}\right)
\le
\bar\Phi\!\left(\frac{r-m}{s}\right)
+
\bar\Phi\!\left(\frac{r-m}{s}\right)
$.
If moreover \(r\ge 2m\), then \(r-m\ge r/2\), so using
\(\bar\Phi(x)\le e^{-x^2/2}\) for \(x\ge0\),
\begin{equation}\label{eq:ex-gaussian-tail-exp}
	\Prob(|W|>r)
	\le
	2\exp\!\left(-\frac{(r-m)^2}{2s^2}\right)
	\le
	2\exp\!\left(-\frac{r^2}{8s^2}\right).
\end{equation}
\end{remark}

\medskip

Apply \eqref{eq:ex-gaussian-tail-exp} with
\begin{equation}
W=Z_t\mid C=c,
\qquad
s=\rho,
\qquad
m=\rho^{t-1}M,
\qquad
r=R_n.
\end{equation}
Set
\begin{equation}
	n_{t,*}:=\max\!\left\{6,\left\lceil(8\rho^{t-1}M)^7\right\rceil\right\}. \label{n*}
\end{equation}
Then for \(n\ge n_{t,*}\) we have \(R_n=\frac14 n^{1/7}\ge 2\rho^{t-1}M\), hence
\begin{equation}\label{eq:ex-feature-tail}
	\Prob(Z_t\notin K_{t,n}\mid C=c)
	\le
	2\exp\!\left(-\frac{R_n^2}{8\rho^2}\right)
	=
	2\exp\!\left(-\frac{n^{2/7}}{128\rho^2}\right).
\end{equation}
The bound is uniform in \(c\in[-M,M]\), so averaging over \(C\) yields
\begin{equation}\label{eq:Ctail}
	\Prob(Z_t\notin K_{t,n})
	\le
	C_t^{\mathrm{tail}}\,e^{-c_t^{\mathrm{tail}}n^{2/7}},
	\qquad
	C_t^{\mathrm{tail}}:=2,
	\quad
	c_t^{\mathrm{tail}}:=\frac{1}{128\rho^2},
	\qquad
	n\ge n_{t,*}.
\end{equation}
This is super-algebraically small in \(n\).

\subsection{Partition, cell probabilities, and the infeasible-mixture bound}
\label{sec:ex-partition}

\subsubsection{Partition and scale parameters}

We now specify the response-space partition and the scale parameters appearing
in the Norets bound \eqref{eq:Bnorets}.

\paragraph{Choice of rate exponent.}
We fix
\begin{equation}\label{eq:ex-beta-choice}
	\beta:=\frac{2}{7},
\end{equation}
which is the exponent given by Corollary~\ref{cor:M0-power-rate} when
\(\theta_t=4\) and \(\xi_t=1\):
\begin{equation}
\beta
=
\frac{1}{2+\frac{1}{\theta_t-2}+\xi_t}
=
\frac{1}{2+\frac12+1}
=
\frac{2}{7}.
\end{equation}
See Remark~\ref{rem:ex-rate-optimization} for the non-optimality of this choice.

\paragraph{Partition resolution and bandwidths.}
For each \(n\ge1\), set \(m_t:=n\) at every step \(t=1,\dots,T\), and define
\begin{equation}\label{eq:ex-scales}
	h_n:=n^{-6/7},
	\qquad
	\delta_n:=n^{-2/7},
	\qquad
	\sigma_n:=n^{-4/7},
	\qquad
	L_n:=\tfrac12 n\,h_n=\tfrac12 n^{1/7}.
\end{equation}
Here \(h_n\) is the cell width, \(\sigma_n\) is the interior Gaussian scale,
\(\delta_n\) is the oscillation scale from \eqref{eq:Bnorets}, and \(L_n\) is
the half-width of the fine partition region.

\paragraph{Response-space partition.}
The interval \([-L_n,L_n)\) is divided into \(n\) cells
\begin{equation}
	A_{t,j}^{(n)}:=[-L_n+(j-1)h_n,\,-L_n+jh_n),
	\qquad j=1,\dots,n,
\end{equation}
with centers
\begin{equation}
	\mu_{t,j}^{(n)}:=-L_n+\Bigl(j-\tfrac12\Bigr)h_n.
\end{equation}
The remainder cell is
\begin{equation}
	A_{t,0}^{(n)}:=\R\setminus[-L_n,L_n).
\end{equation}

\paragraph{Tail-component parameters.}
We take \(r_{t-1}=1\), as in \S\ref{sec:ex-M0-assumptions}, and
\begin{equation}
	\sigma_{t,0}:=\sqrt{\frac{2}{\pi}}.
\end{equation}
With this choice, the third condition in \eqref{eq:delta-sigma-h-condition}
holds with equality.

\paragraph{Verification of \eqref{eq:delta-sigma-h-condition}.}
Since \(D=1\), the three conditions in \eqref{eq:delta-sigma-h-condition} read:
\begin{enumerate}[label=(\roman*)]
	\item $\sigma_n/\delta_n=n^{-4/7}/n^{-2/7}=n^{-2/7}\to 0$.
	\item Since $D=1$, we have $\delta_n^{D-1}=1$, so the second condition is
	$h_n/\sigma_n=n^{-6/7}/n^{-4/7}=n^{-2/7}\to 0$.
	\item $(r_{t-1}/2)^D/(2\pi\sigma_{t,0}^2)^{D/2}
	=\frac{1/2}{\sqrt{2\pi\cdot 2/\pi}}
	=\frac{1/2}{2}=\frac14=2^{-(D+1)}$.
\end{enumerate}
Also, the boundary condition in \eqref{eq:delta-sigma-h-condition} is
eventually satisfied because \(L_n=\frac12 n^{1/7}\to\infty\), so every fixed
\(\delta_n\)-cube is eventually contained in the fine region.

\subsubsection{Cell probabilities and log-odds}

Write
\begin{equation}
a_{j,n}:=-L_n+(j-1)h_n,
\qquad
b_{j,n}:=-L_n+jh_n
\end{equation}
for the endpoints of \(A_{t,j}^{(n)}\).  Since the reverse-step density is
\(f_t(y\mid z)=\varphi(y;z,\sigma^2)\) by \eqref{eq:ex-ft}, the cell
probabilities are
\begin{align}
	G_{t,j}^{(n)}(z)
	&:=
	\int_{A_{t,j}^{(n)}} f_t(y\mid z)\,\dd y
	=
	\Phi\!\left(\frac{b_{j,n}-z}{\sigma}\right)
	-
	\Phi\!\left(\frac{a_{j,n}-z}{\sigma}\right),
	\qquad j=1,\dots,n,
	\label{eq:ex-Gtj}\\
	G_{t,0}^{(n)}(z)
	&:=
	\int_{A_{t,0}^{(n)}} f_t(y\mid z)\,\dd y
	=
	\Phi\!\left(\frac{-L_n-z}{\sigma}\right)
	+
	1-\Phi\!\left(\frac{L_n-z}{\sigma}\right).
	\label{eq:ex-Gt0}
\end{align}
The exact step-\(t\) log-odds are therefore
\begin{equation}\label{eq:ex-log-odds}
	\Lambda_{t,j}^{n}(z)
	:=
	\log G_{t,j}^{(n)}(z)-\log G_{t,0}^{(n)}(z),
	\qquad j=1,\dots,n.
\end{equation}
Since \(f_t(\cdot\mid z)>0\) everywhere, every \(G_{t,j}^{(n)}(z)\) is strictly
positive, and because \(\Phi\) is \(C^\infty\), the functions
\(\Lambda_{t,j}^{n}\) are \(C^\infty\) on \(K_{t,n}\).  Thus
Assumption~\ref{ass:smooth-logodds} holds with $s_t$ arbitrarily large; we will only need the case \(s_t=1\).

\subsubsection{Term-by-term computation of
	\texorpdfstring{$B_{t,n}^{\mathrm{Nor}}$}{B\^Nor}}

We now evaluate each term in \eqref{eq:Bnorets} for \(D=1\).

\paragraph{Term 1: the local oscillation integral \(L_t(\delta_n)\).}
Using the \(C^1\)-branch of the local oscillation operator \eqref{eq:Duft} and
\(f_t(y\mid z)=\varphi(y;z,\sigma^2)\) from \eqref{eq:ex-ft},
$
\partial_u\log f_t(u\mid z)
=
-\frac{u-z}{\sigma^2}
$,
so for \(C_{\delta_n}(y)=[y-\delta_n/2,y+\delta_n/2]\),
\begin{equation}
	\mathcal D_{\delta_n}f_t(y\mid z)
	=
	\frac{\delta_n}{2}
	\sup_{u\in C_{\delta_n}(y)}
	\bigl|\partial_u\log f_t(u\mid z)\bigr|
	\le
	\frac{\delta_n}{2\sigma^2}
	\left(|y-z|+\frac{\delta_n}{2}\right).
\end{equation}
Integrating against the joint law \(F_t\) from \eqref{eq:ex-Ft}, and using
\(\E|Y-Z|=\sigma\sqrt{2/\pi}\) from \eqref{eq:ex-Ft-moments}, we obtain
\begin{equation}\label{eq:ex-Lt}
	L_t(\delta_n)
	\le
	\frac{\delta_n}{2\sigma^2}\,\E|Y-Z|
	+
	\frac{\delta_n^2}{4\sigma^2}
	=
	\frac{1}{2\sigma}\sqrt{\frac{2}{\pi}}\,n^{-2/7}
	+
	\frac{1}{4\sigma^2}\,n^{-4/7}.
\end{equation}

\paragraph{Term 2: the boundary oscillation integral \(T_{t,n}(\delta_n)\).}
By definition, the \(\delta_n\)-boundary of the remainder cell \(A_{t,0}^{(n)}\)
consists of those \(y\) for which \(C_{\delta_n}(y)\) meets
\(\R\setminus[-L_n,L_n)\).  Equivalently,
\begin{equation}
	\mathcal B_{t,n}(\delta_n)
	=
	\Bigl\{(z,y):C_{\delta_n}(y)\cap A_{t,0}^{(n)}\neq\emptyset\Bigr\}
	=
	\{|y|>L_n-\delta_n/2\}.
\end{equation}
Set
\begin{equation}
	a_n:=L_n-\delta_n/2
	=
	\frac12 n^{1/7}-\frac12 n^{-2/7}.
\end{equation}
Let \(Y=X_{t-1}\).  Since \(Y\mid C=c\sim\mathcal N(\rho^{t-1}c,1)\), the
template \eqref{eq:ex-gaussian-tail-exp} applies with
\begin{equation}
W=Y\mid C=c,
\qquad
s=1,
\qquad
m=\rho^{t-1}M,
\qquad
r=a_n.
\end{equation}
For \(n\ge n_{t,*}\) we have
$
a_n
=
\frac12 n^{1/7}-\frac12 n^{-2/7}
\ge
\frac14 n^{1/7}
\ge
2\rho^{t-1}M
$,
so
\begin{equation}\label{eq:ex-boundary-prob}
	p_{t,n}
	:=
	\Prob(|Y|>a_n)
	\le
	2\exp\!\left(-\frac{a_n^2}{8}\right)
	\le
	2\exp\!\left(-\frac{n^{2/7}}{128}\right),
	\qquad n\ge n_{t,*}.
\end{equation}

Again from \eqref{eq:Duft}, now with \(r_{t-1}=1\),
\begin{equation}
	\mathcal D_1f_t(y\mid z)
	=
	\frac12\sup_{u\in[y-\frac12,y+\frac12]}
	\bigl|\partial_u\log f_t(u\mid z)\bigr|
	\le
	\frac{|y-z|+\frac12}{2\sigma^2}.
\end{equation}
Hence
\begin{equation}
	T_{t,n}(\delta_n)
	\le
	\frac{1}{2\sigma^2}
	\int_{\mathcal B_{t,n}(\delta_n)}
	\left(|y-z|+\frac12\right)\,\dd F_t.
\end{equation}
Using Cauchy--Schwarz and \(\E[(Y-Z)^2]=\sigma^2\) from
\eqref{eq:ex-Ft-moments},
\begin{equation}
\int |y-z|\,\mathbf 1_{\mathcal B_{t,n}(\delta_n)}\,\dd F_t
\le
\bigl(\E[(Y-Z)^2]\bigr)^{1/2}\,p_{t,n}^{1/2}
=
\sigma\,p_{t,n}^{1/2}.
\end{equation}
Also \(p_{t,n}\le p_{t,n}^{1/2}\) because \(p_{t,n}\le1\).  Therefore,
for \(n\ge n_{t,*}\),
\begin{equation}\label{eq:ex-Tt}
	T_{t,n}(\delta_n)
	\le
	\frac{\sigma+\frac12}{2\sigma^2}\,p_{t,n}^{1/2}
	\le
	C_t^{(T)}\,e^{-c_t^{(T)}n^{2/7}},
	\quad \text{with} \quad
	C_t^{(T)}:=\frac{\sigma+\frac12}{\sqrt2\,\sigma^2},
	\quad
	c_t^{(T)}:=\frac{1}{256}.
\end{equation}

\paragraph{Term 3: the tail-component integral \(U_{t,n}(\delta_n)\).}
For \(D=1\), \(r_{t-1}=1\), and \(\sigma_{t,0}=\sqrt{2/\pi}\), the second term
in \(U_{t,n}\) from \eqref{eq:LTU} is
$
-\log\frac{(r_{t-1}/2)^D}{(2\pi\sigma_{t,0}^2)^{D/2}}
=
-\log\frac{1/2}{\sqrt{2\pi\cdot 2/\pi}}
=
\log 4
$.
Thus
\begin{equation}
	U_{t,n}(\delta_n)
	=
	\frac{\pi}{4}
	\int_{\mathcal B_{t,n}(\delta_n)} y^2\,\dd F_t
	+
	(\log 4)\,p_{t,n}.
\end{equation}
Since \(\mathcal B_{t,n}(\delta_n)\) depends only on \(Y=X_{t-1}\),
\begin{equation}
\int_{\mathcal B_{t,n}(\delta_n)} y^2\,\dd F_t
=
\E\!\left[Y^2\,\mathbf 1_{\{|Y|>a_n\}}\right]
\le
\bigl(\E[Y^4]\bigr)^{1/2}\,p_{t,n}^{1/2}
\end{equation}
by Cauchy--Schwarz.  Now \(Y\mid C=c\sim\mathcal N(\rho^{t-1}c,1)\), so
\begin{equation}
\E[Y^4\mid C=c]
=
(\rho^{t-1}c)^4+6(\rho^{t-1}c)^2+3
\le
(\rho^{t-1}M)^4+6(\rho^{t-1}M)^2+3.
\end{equation}
Hence
\begin{equation}
	\bigl(\E[Y^4]\bigr)^{1/2}
	\le
	\sqrt{(\rho^{t-1}M)^4+6(\rho^{t-1}M)^2+3}.
\end{equation}
Using the already established bound \eqref{eq:ex-boundary-prob},
\begin{equation}
p_{t,n}^{1/2}
\le
\sqrt2\,e^{-n^{2/7}/256},
\qquad
p_{t,n}\le p_{t,n}^{1/2},
\qquad
n\ge n_{t,*},
\end{equation}
and therefore
\begin{equation}\label{eq:ex-Ut}
	U_{t,n}(\delta_n)
	\le
	C_t^{(U)}\,e^{-c_t^{(U)}n^{2/7}},
	\qquad n\ge n_{t,*},
\end{equation}
where
\begin{equation}
	C_t^{(U)}
	:=
	\sqrt2\left(
	\frac{\pi}{4}\sqrt{(\rho^{t-1}M)^4+6(\rho^{t-1}M)^2+3}
	+\log 4
	\right),
	\qquad
	c_t^{(U)}:=\frac{1}{256}.
\end{equation}

\paragraph{Term 4: the Riemann-sum error.}
With \(D=1\), the fourth term in \eqref{eq:Bnorets} becomes
\begin{equation}
	\frac{6D^{3/2}\delta_n^{D-1}h_n}{(2\pi)^{D/2}\sigma_n^D}
	=
	\frac{6h_n}{\sqrt{2\pi}\,\sigma_n}
	=
	\frac{6}{\sqrt{2\pi}}\,n^{-2/7}. \label{eq:ex-riemann}
\end{equation}

\paragraph{Term 5: the Gaussian tail term.}
Again with \(\delta_n,\sigma_n\) from \eqref{eq:ex-scales},
\begin{equation}\label{eq:ex-gauss-tail}
	2\exp\!\left(-\frac{(\delta_n/\sigma_n)^2}{8}\right)
	=
	2\exp\!\left(-\frac{n^{4/7}}{8}\right).
\end{equation}
This decays faster than \(e^{-c n^{2/7}}\); in particular,
\begin{equation}
2e^{-n^{4/7}/8}
\le
2e^{-n^{2/7}/256},
\qquad n\ge1.
\end{equation}

\subsubsection{Assembled Norets bound and rate}

Combining \eqref{eq:ex-Lt}, \eqref{eq:ex-Tt}, \eqref{eq:ex-Ut}, \eqref{eq:ex-riemann}, and \eqref{eq:ex-gauss-tail}, we obtain
\begin{equation}\label{eq:ex-Bnor}
	B_{t,n}^{\mathrm{Nor}}
	\le
	\underbrace{\left(
		\frac{1}{2\sigma}\sqrt{\frac{2}{\pi}}
		+
		\frac{6}{\sqrt{2\pi}}
		\right)}_{=:\,\kappa_1(\sigma)}
	n^{-2/7}
	+
	\frac{1}{4\sigma^2}\,n^{-4/7}
	+
	C_t^\sharp e^{-c_t^\sharp n^{2/7}},
\end{equation}
for all \(n\ge n_{t,*}\), where one may take
\begin{equation}
	C_t^\sharp
	:=
	\frac{\sigma+\frac12}{\sqrt2\,\sigma^2}
	+
	\sqrt2\left(
	\frac{\pi}{4}\sqrt{(\rho^{t-1}M)^4+6(\rho^{t-1}M)^2+3}
	+\log 4
	\right)
	+2,
	\qquad
	c_t^\sharp:=\frac{1}{256}.
\end{equation}
The leading algebraic coefficient
\begin{equation}
	\kappa_1(\sigma)
	=
	\frac{1}{2\sigma}\sqrt{\frac{2}{\pi}}
	+
	\frac{6}{\sqrt{2\pi}}
\end{equation}
is decreasing in \(\sigma=\sqrt{1-\rho^2}\).  In the boundary case
\(\rho\to0\) (so \(\sigma\to1\)),
$
\kappa_1(1)
=
\frac12\sqrt{\frac{2}{\pi}}
+
\frac{6}{\sqrt{2\pi}}
\approx 2.79
$.
Thus the dominant contribution to \(B_{t,n}^{\mathrm{Nor}}\) is still the
algebraic term \(\kappa_1(\sigma)n^{-2/7}\); all remaining terms are strictly
smaller as \(n\to\infty\). By Proposition~\ref{prop:M0-explicit-bound}, \eqref{eq:ex-Bnor} bounds the step-$t$ KL error for the infeasible model, more specifically, for $n\ge n_{t,*}$ we find
\begin{equation}\label{eq:ex-M0-rate}
	d_{\KL,t}\!\left(\tilde p_t^{\mathcal M_0,n}\right)
	\le
	C_t^{\mathrm{Nor}}\,n^{-2/7},
\end{equation}
where $C_t^{\mathrm{Nor}}$ can be any constant satisfying:
\begin{equation}
	\kappa_1(\sigma)+\frac{1}{4\sigma^2}\,n_{t,*}^{-2/7}
	+
	\sup_{n \ge n_{t,*}}C_t^\sharp e^{-c_t^\sharp n^{2/7}} n^{2/7}
	\le
	C_t^{\mathrm{Nor}}.
\end{equation}
Comparing \eqref{eq:ex-M0-rate} with \eqref{eq:M0-power-rate} we see the decay rate of the per-step error of this toy model to be  $\beta_t = \beta = 2/7$.

\begin{remark}[Optimality of the rate exponent]\label{rem:ex-rate-optimization}
	The choice \(\theta_t=4\), \(\xi_t=1\), hence \(\beta=2/7\), is convenient but
	not optimal.  Since the joint law \(F_t\) in \eqref{eq:ex-Ft} is a bounded
	location-mixture of Gaussians, all of its polynomial moments are finite.  Thus
	\(\theta_t\) may be taken arbitrarily large.  With \(\theta_t=q\) and
	\(\xi_t\to 0\), Corollary~\ref{cor:M0-power-rate} gives
	\begin{equation}
		\beta_t
		=
		\frac{1}{2+\frac{1}{q-2}+\xi_t}
		\longrightarrow
		\frac12
		\qquad\text{as }q\to\infty,\ \xi_t\to0.
	\end{equation}
	However, the constant in Corollary~\ref{cor:M0-power-rate} depends on the
	\(q\)-th moment of \(F_t\), and that constant deteriorates as \(q\to\infty\).
	So the limit \(\beta_t\to\frac12\) should be understood as an asymptotic
	optimization of the exponent, not as a uniform bound with fixed constants.
\end{remark}

\subsection{Neural approximation and final bound}
\label{sec:ex-neural}

\subsubsection{Log-odds complexity}

To verify Assumption~\ref{ass:smooth-logodds} in this example, we use the
growing compact window \(K_{t,n}=[-R_n,R_n]\) from \eqref{eq:ex-Ktn} and the
affine parametrization
\begin{equation}
	\Psi_{t,n}:[0,1]\to K_{t,n},
	\qquad
	\Psi_{t,n}(u):=R_n(2u-1).
\end{equation}
Recall also that the scales \(h_n,L_n\) are given by \eqref{eq:ex-scales}, the
step-\(t\) reverse density by \eqref{eq:ex-ft}, the cell probabilities by
\eqref{eq:ex-Gtj}--\eqref{eq:ex-Gt0}, and the exact log-odds by
\eqref{eq:ex-log-odds}. We set
\begin{equation}
	\Lambda_{t,n}^{\max}
	:=
	\max_{1\le j\le n}
	\|\Lambda_{t,j}^{n}\circ\Psi_{t,n}\|_{C^1([0,1])}.
\end{equation}

\paragraph{Pointwise bound on the log-odds.}
Fix \(z\in K_{t,n}\), i.e.\ \(|z|\le R_n\). For any interior cell
\(A_{t,j}^{(n)}\) from the partition in \S\ref{sec:ex-partition},
\begin{equation}
	G_{t,j}^{(n)}(z)
	=
	\int_{A_{t,j}^{(n)}} f_t(y\mid z)\,\dd y
	\le
	h_n \sup_{y\in\R} f_t(y\mid z)
	=
	\frac{h_n}{\sigma\sqrt{2\pi}},
\end{equation}
because \(f_t(\,\cdot\mid z)=\varphi(\,\cdot\,;z,\sigma^2)\) by
\eqref{eq:ex-ft}. On the other hand,
\begin{equation}
	G_{t,j}^{(n)}(z)
	\ge
	h_n\inf_{y\in A_{t,j}^{(n)}} f_t(y\mid z)
	\ge
	\frac{h_n}{\sigma\sqrt{2\pi}}
	\exp\!\left(-\frac{(L_n+R_n+h_n)^2}{2\sigma^2}\right),
\end{equation}
since \(A_{t,j}^{(n)}\subset [-L_n,L_n)\) and hence
\(|y-z|\le L_n+R_n+h_n\) on \(A_{t,j}^{(n)}\times K_{t,n}\). Therefore
\begin{equation}
	-\log G_{t,j}^{(n)}(z)
	\le
	\frac{(L_n+R_n+h_n)^2}{2\sigma^2}
	-\log h_n
	+
	C.
\end{equation}

For the tail cell \(A_{t,0}^{(n)}\), \eqref{eq:ex-Gt0} gives
\begin{equation}
	G_{t,0}^{(n)}(z)
	=
	\Phi\!\left(\frac{-L_n-z}{\sigma}\right)
	+
	1-\Phi\!\left(\frac{L_n-z}{\sigma}\right)
	=
	\bar\Phi\!\left(\frac{L_n+z}{\sigma}\right)
	+
	\bar\Phi\!\left(\frac{L_n-z}{\sigma}\right).
\end{equation}
As a function of \(z\), the right-hand side is even, and for \(z\ge0\) its derivative is
\begin{equation}
	\frac{1}{\sigma}
	\left[
	\varphi\!\left(\frac{L_n-z}{\sigma};0,1\right)
	-
	\varphi\!\left(\frac{L_n+z}{\sigma};0,1\right)
	\right]
	\ge 0.
\end{equation}
Hence \(G_{t,0}^{(n)}(z)\) is minimized at \(z=0\), so for all \(|z|\le R_n\),
\begin{equation}
	G_{t,0}^{(n)}(z)
	\ge
	2\,\bar\Phi\!\left(\frac{L_n}{\sigma}\right).
\end{equation}
Using the lower Mills bound
\(\bar\Phi(x)\ge \frac{x}{1+x^2}\varphi(x;0,1)\) for \(x>0\), we obtain
\begin{equation}
	-\log G_{t,0}^{(n)}(z)
	\le
	\frac{L_n^2}{2\sigma^2}
	+
	C\log(1+L_n)
	+
	C.
\end{equation}

Since
\begin{equation}
\Lambda_{t,j}^{n}(z)=\log G_{t,j}^{(n)}(z)-\log G_{t,0}^{(n)}(z),
\end{equation}
we have
\begin{equation}
|\Lambda_{t,j}^{n}(z)|
\le
-\log G_{t,j}^{(n)}(z)-\log G_{t,0}^{(n)}(z).
\end{equation}
Using \eqref{eq:ex-scales} and \eqref{eq:ex-Ktn}, namely
\(h_n=n^{-6/7}\), \(L_n=\frac12 n^{1/7}\), and \(R_n=\frac14 n^{1/7}\), this yields
\begin{equation}\label{eq:ex-logit-sup}
	\sup_{|z|\le R_n}\max_{1\le j\le n}
	|\Lambda_{t,j}^{n}(z)|
	\le
	C_{t,1}n^{2/7}+C_{t,2}\log n
	\le
	C_{t,3}n^{2/7}.
\end{equation}

\paragraph{Derivative bound.}
For \(j=1,\dots,n\), differentiating under the integral sign in
\eqref{eq:ex-Gtj} is justified because \(f_t(\cdot\mid z)\) is Gaussian
\eqref{eq:ex-ft}. Since
$
\partial_z f_t(y\mid z)=\frac{y-z}{\sigma^2}f_t(y\mid z)
$,
we get
\begin{equation}
	\partial_z G_{t,j}^{(n)}(z)
	=
	\frac{1}{\sigma^2}
	\int_{A_{t,j}^{(n)}}(y-z)f_t(y\mid z)\,\dd y,
\end{equation}
hence
\begin{equation}
	\partial_z \log G_{t,j}^{(n)}(z)
	=
	\frac{1}{\sigma^2}
	\frac{\int_{A_{t,j}^{(n)}}(y-z)f_t(y\mid z)\,\dd y}
	{\int_{A_{t,j}^{(n)}}f_t(y\mid z)\,\dd y}.
\end{equation}
Since \(|y-z|\le L_n+R_n+h_n\) on \(A_{t,j}^{(n)}\times K_{t,n}\),
\begin{equation}
	\sup_{|z|\le R_n}\max_{1\le j\le n}
	\left|\partial_z \log G_{t,j}^{(n)}(z)\right|
	\le
	\frac{L_n+R_n+h_n}{\sigma^2}.
\end{equation}

For the tail cell, differentiating \eqref{eq:ex-Gt0} gives
\begin{equation}
	\partial_z G_{t,0}^{(n)}(z)
	=
	\frac{1}{\sigma}
	\left[
	\varphi\!\left(\frac{L_n-z}{\sigma};0,1\right)
	-
	\varphi\!\left(\frac{L_n+z}{\sigma};0,1\right)
	\right].
\end{equation}
By symmetry it suffices to consider \(z\in[0,R_n]\). Then
\begin{equation}
	0\le \partial_z G_{t,0}^{(n)}(z)
	\le
	\frac{1}{\sigma}
	\varphi\!\left(\frac{L_n-z}{\sigma};0,1\right),
	\qquad
	G_{t,0}^{(n)}(z)\ge
	\bar\Phi\!\left(\frac{L_n-z}{\sigma}\right).
\end{equation}
Therefore
\begin{equation}
	0\le
	\partial_z \log G_{t,0}^{(n)}(z)
	\le
	\frac{1}{\sigma}\,
	\frac{\varphi((L_n-z)/\sigma;0,1)}{\bar\Phi((L_n-z)/\sigma)}.
\end{equation}
Using the upper Mills bound
\(\varphi(x;0,1)/\bar\Phi(x)\le x+x^{-1}\) for \(x>0\), together with
\(L_n-z\ge L_n-R_n=\frac14 n^{1/7}\), we obtain
\begin{equation}\label{eq:ex-logit-deriv}
	\sup_{|z|\le R_n}
	\left|\partial_z \log G_{t,0}^{(n)}(z)\right|
	\le
	C_{t,4}(1+L_n)
	\le
	C_{t,5}n^{1/7}.
\end{equation}

Combining the interior and tail estimates,
\begin{equation}
	\sup_{|z|\le R_n}\max_{1\le j\le n}
	|\partial_z\Lambda_{t,j}^{n}(z)|
	\le
	C_{t,6}n^{1/7}.
\end{equation}
Since \(\Psi_{t,n}'(u)=2R_n\),
\begin{equation}
	\|\Lambda_{t,j}^{n}\circ\Psi_{t,n}\|_{C^1([0,1])}
	\le
	\sup_{|z|\le R_n}|\Lambda_{t,j}^{n}(z)|
	+
	2R_n\sup_{|z|\le R_n}|\partial_z\Lambda_{t,j}^{n}(z)|,
\end{equation}
and therefore \eqref{eq:ex-logit-sup}, \eqref{eq:ex-logit-deriv}, and
\(R_n=\frac14 n^{1/7}\) imply
\begin{equation}\label{eq:ex-Lambda-growth}
	\Lambda_{t,n}^{\max}\le C_t^\Lambda n^{2/7}.
\end{equation}

\subsubsection{Neural logit approximation}

Apply Proposition~\ref{prop:nn-logit-rate} with \(d_t=s_t=1\) to the
one-dimensional functions \(\Lambda_{t,j}^{n}\circ\Psi_{t,n}\). This gives a
ReLU network \(\bar a_t^{\vartheta_t^{(n)}}:\R\to\R^n\) such that
\begin{equation}
	\max_{1\le j\le n}\sup_{u\in[0,1]}
	\left|
	\bar a_{t,j}^{\vartheta_t^{(n)}}(u)
	-
	\bigl(\Lambda_{t,j}^{n}\circ\Psi_{t,n}\bigr)(u)
	\right|
	\le
	1360\,\Lambda_{t,n}^{\max}\,\mathscr N_t^{-2}\mathscr L_t^{-2}.
\end{equation}
Now compose with the affine inverse
\begin{equation}
	\Psi_{t,n}^{-1}(z)=\frac{z/R_n+1}{2},
	\qquad z\in K_{t,n},
\end{equation}
and define
\begin{equation}
	a_t^{\vartheta_t^{(n)}}(z)
	:=
	\bar a_t^{\vartheta_t^{(n)}}\!\bigl(\Psi_{t,n}^{-1}(z)\bigr).
\end{equation}
Then
\begin{equation}
	\eta_t^{\NN}
	:=
	\max_{1\le j\le n}\sup_{z\in K_{t,n}}
	\left|
	a_{t,j}^{\vartheta_t^{(n)}}(z)-\Lambda_{t,j}^{n}(z)
	\right|
	\le
	1360\,\Lambda_{t,n}^{\max}\,\mathscr N_t^{-2}\mathscr L_t^{-2}.
\end{equation}
Using \eqref{eq:ex-Lambda-growth}, it is enough to choose
\begin{equation}\label{eq:ex-network-size}
	\mathscr N_t^{(n)}=\mathscr L_t^{(n)}
	:=
	\left\lceil
	(5440\,T\,C_t^\Lambda)^{1/4}\,n^{9/28}
	\right\rceil
\end{equation}
to ensure
\begin{equation}
	\eta_t^{\NN}\le \frac{1}{4Tn}.
\end{equation}

Because the feature \(Z_t=\phi_t(C,X_t)\) from \eqref{eq:ex-feature} is not
compactly supported, we clip it to the window \(K_{t,n}\) in
\eqref{eq:ex-Ktn}:
\begin{equation}
	\Pi_n(z):=\max\{-R_n,\min\{z,R_n\}\}.
\end{equation}
The map \(\Pi_n\) is piecewise affine, hence can be implemented exactly by a
small ReLU subnetwork and absorbed into the first layer. Using the baseline
softmax weights \(\alpha_{t,j}[\cdot]\) from \eqref{eq:generic-softmax} and the
Gaussian-mixture template \eqref{eq:generic-logit-mixture}, we define
\begin{equation}\label{eq:ex-neural-mixture}
	\tilde p_{t,\vartheta_t^{(n)}}^{\NN}(y\mid z)
	:=
	\sum_{j=1}^{n}
	\alpha_{t,j}\bigl[a_t^{\vartheta_t^{(n)}}\bigr]\!\left(\Pi_n(z)\right)
	\varphi(y;\mu_{t,j}^{(n)},\sigma_n^2)
	+
	\alpha_{t,0}\bigl[a_t^{\vartheta_t^{(n)}}\bigr]\!\left(\Pi_n(z)\right)
	\varphi(y;0,\sigma_{t,0}^2),
\end{equation}
and the induced step-\(t\) reverse kernel by
\begin{equation}
	p_{t,\vartheta_t^{(n)}}^{\NN}(\dd y\mid c,x_t)
	:=
	\tilde p_{t,\vartheta_t^{(n)}}^{\NN}
	(\dd y\mid \phi_t(c,x_t)).
\end{equation}

On the event \(\{|Z_t|\le R_n\}\), clipping is inactive, so
\(\Pi_n(Z_t)=Z_t\in K_{t,n}\). By Lemma~\ref{lem:baseline-softmax-stability},
\begin{equation}
	\KL\!\left(
	f_t(\cdot\mid z)\,\middle\|\,
	\tilde p_{t,\vartheta_t^{(n)}}^{\NN}(\cdot\mid z)
	\right)
	\le
	\KL\!\left(
	f_t(\cdot\mid z)\,\middle\|\,
	\tilde p_t^{\mathcal M_0,n}(\cdot\mid z)
	\right)
	+
	2\eta_t^{\NN}
\end{equation}
for every \(z\in K_{t,n}\). Integrating over \(\{|z|\le R_n\}\), using the
infeasible-mixture bound \eqref{eq:ex-M0-rate} and
\(\eta_t^{\NN}\le 1/(4Tn)\), we get
\begin{equation}\label{eq:ex-neural-rate}
	\int_{|z|\le R_n}
	\KL\!\left(
	f_t(\cdot\mid z)\,\middle\|\,
	\tilde p_{t,\vartheta_t^{(n)}}^{\NN}(\cdot\mid z)
	\right)
	(F_t)_Z(\dd z)
	\le
	C_t^{\mathrm{Nor}}\,n^{-2/7}
	+
	\frac{1}{2Tn}.
\end{equation}

\subsubsection{Tail contribution and assembled bound}

It remains to control the contribution of \(\{|Z_t|>R_n\}\). For such \(z\),
let \(u:=\Pi_n(z)\in\{-R_n,R_n\}\). Since
\(|a_{t,j}^{\vartheta_t^{(n)}}(u)-\Lambda_{t,j}^{n}(u)|\le\eta_t^{\NN}\) for
all \(j\),
\begin{equation}
	1+\sum_{j=1}^{n}e^{a_{t,j}^{\vartheta_t^{(n)}}(u)}
	\le
	e^{\eta_t^{\NN}}
	\left(
	1+\sum_{j=1}^{n}e^{\Lambda_{t,j}^{n}(u)}
	\right).
\end{equation}
Hence, using \(G_{t,0}^{(n)}=\alpha_{t,0}[\Lambda_t^n]\),
\begin{equation}
	\alpha_{t,0}\bigl[a_t^{\vartheta_t^{(n)}}\bigr](u)
	=
	\frac{1}{1+\sum_{j=1}^{n}e^{a_{t,j}^{\vartheta_t^{(n)}}(u)}}
	\ge
	e^{-\eta_t^{\NN}}\,G_{t,0}^{(n)}(u)
	\ge
	e^{-\eta_t^{\NN}}
	\bar\Phi\!\left(\frac{L_n-R_n}{\sigma}\right).
\end{equation}
Therefore, by \eqref{eq:ex-neural-mixture},
\begin{equation}
	\tilde p_{t,\vartheta_t^{(n)}}^{\NN}(y\mid z)
	\ge
	e^{-\eta_t^{\NN}}
	\bar\Phi\!\left(\frac{L_n-R_n}{\sigma}\right)\,
	\varphi(y;0,\sigma_{t,0}^2).
\end{equation}

Now \(f_t(\cdot\mid z)=\mathcal N(z,\sigma^2)\) by \eqref{eq:ex-ft}. If
\(q(y)\ge \beta r(y)\) pointwise, then
\(\KL(p\|q)\le \KL(p\|r)-\log\beta\). Applying this with
\begin{equation}
p=\mathcal N(z,\sigma^2),
\qquad
q=\tilde p_{t,\vartheta_t^{(n)}}^{\NN}(\cdot\mid z),
\qquad
r=\mathcal N(0,\sigma_{t,0}^2),
\end{equation}
we obtain, for \(|z|>R_n\),
\begin{align}
	\KL\!\left(
	f_t(\cdot\mid z)\,\middle\|\,
	\tilde p_{t,\vartheta_t^{(n)}}^{\NN}(\cdot\mid z)
	\right)
	&\le
	\KL\!\left(
	\mathcal N(z,\sigma^2)\,\middle\|\,
	\mathcal N(0,\sigma_{t,0}^2)
	\right)
	+
	\eta_t^{\NN}
	-
	\log\bar\Phi\!\left(\frac{L_n-R_n}{\sigma}\right).
\end{align}
The Gaussian KL formula gives a bound of the form \(C_t(1+z^2)\), while the
same lower-Mills estimate used earlier yields
\(-\log\bar\Phi((L_n-R_n)/\sigma)\le C_t' n^{2/7}\). Hence
\begin{equation}\label{eq:ex-pointwise-tail-kl}
	\KL\!\left(
	f_t(\cdot\mid z)\,\middle\|\,
	\tilde p_{t,\vartheta_t^{(n)}}^{\NN}(\cdot\mid z)
	\right)
	\le
	C_{t,7}\bigl(1+z^2+n^{2/7}\bigr),
	\qquad |z|>R_n.
\end{equation}

Integrating \eqref{eq:ex-pointwise-tail-kl} over \(\{|z|>R_n\}\) gives
\begin{equation}\begin{gathered}
	\int_{|z|>R_n}
	\KL\!\left(
	f_t(\cdot\mid z)\,\middle\|\,
	\tilde p_{t,\vartheta_t^{(n)}}^{\NN}(\cdot\mid z)
	\right)
	(F_t)_Z(\dd z)
	\\
	\le
	C_{t,7}
	\left[
	(1+n^{2/7})\Prob(|Z_t|>R_n)
	+
	\int_{|z|>R_n} z^2\,(F_t)_Z(\dd z)
	\right].
\end{gathered}\end{equation}

The first term is already controlled by the feature-tail bound
\eqref{eq:Ctail}. For the second term, use exactly the same
Cauchy--Schwarz step as in the earlier Gaussian tail estimates:
\begin{equation}
	\int_{|z|>R_n} z^2\,(F_t)_Z(\dd z)
	\le
	\bigl(\E[Z_t^4]\bigr)^{1/2}\Prob(|Z_t|>R_n)^{1/2}.
\end{equation}
Since \(Z_t\mid C=c\sim\mathcal N(\rho^{t-1}c,\rho^2)\) from
\eqref{eq:ex-feature} and \eqref{eq:ex-marginal}, its fourth moment is uniformly
bounded:
\begin{equation}
	\E[Z_t^4]
	\le
	(\rho^{t-1}M)^4
	+
	6\rho^2(\rho^{t-1}M)^2
	+
	3\rho^4
	=:M_{t,4}^{(Z)}.
\end{equation}
Combining this with \eqref{eq:Ctail}, and absorbing the prefactor
\(1+n^{2/7}\) into the exponential after decreasing the rate constant if
necessary, we obtain constants
\(\widetilde C_t^{\mathrm{tail}},\widetilde c_t^{\mathrm{tail}}>0\) such that
\begin{equation}\label{eq:ex-neural-tail}
\int_{|z|>R_n}
\KL\!\left(
f_t(\cdot\mid z)\,\middle\|\,
\tilde p_{t,\vartheta_t^{(n)}}^{\NN}(\cdot\mid z)
\right)
(F_t)_Z(\dd z)
\le
\widetilde C_t^{\mathrm{tail}}
e^{-\widetilde c_t^{\mathrm{tail}} n^{2/7}}.
\end{equation}

Combining this with the terminal mismatch \eqref{eq:ex-terminal}, the
on-window estimate \eqref{eq:ex-neural-rate}, and summing over \(t=1,\dots,T\),
we obtain
\begin{equation}\label{eq:ex-final}
	\mathcal E_{\KL}\!\left(\pi_{p^{\NN,n},\nu}\right)
	\le
	\frac{\rho^{2T}M^2}{6}
	+
	\sum_{t=1}^T C_t^{\mathrm{Nor}}\,n^{-2/7}
	+
	\frac{1}{2n}
	+
	\sum_{t=1}^T
	\widetilde C_t^{\mathrm{tail}}
	e^{-\widetilde c_t^{\mathrm{tail}} n^{2/7}}.
\end{equation}
It is useful to compare this concrete
estimate with the abstract finite-resolution bound \eqref{eq:global-kl-power}.
The latter has the form
\begin{equation}
\mathcal E_{\KL}(\pi_{p^{\NN},\nu})
\;\le\;
\E[\Comp(C)]
+
\sum_{t=1}^T
\Bigl(
C_t^{\mathrm{Nor}}\,m_t^{-\beta_t}
+
170(s_t+1)^{d_t}8^{s_t}\,
\Lambda_{t,m_t}^{\max}\,
\mathscr N_t^{-2s_t/d_t}\mathscr L_t^{-2s_t/d_t}
\Bigr),
\end{equation}
where the three ingredients are, respectively: terminal mismatch, stepwise
Gaussian-mixture approximation error, and stepwise neural logit approximation
error.

In the present Gaussian--OU example, these three abstract terms become
completely explicit.  First, the terminal mismatch is exactly
\(\E[\Comp(C)]=\rho^{2T}M^2/6\) by \eqref{eq:ex-terminal}.  Second, we take the
same mixture resolution \(m_t=n\) at every step, and with the choice
\(\theta_t=4\), \(\xi_t=1\) from \eqref{eq:ex-beta-choice}, Corollary
\ref{cor:M0-power-rate} gives \(\beta_t=2/7\), so the abstract mixture term
reduces to \(\sum_{t=1}^T C_t^{\mathrm{Nor}} n^{-2/7}\).  Third, here
\(d_t=s_t=1\), and the log-odds complexity satisfies
\(\Lambda_{t,n}^{\max}=O(n^{2/7})\) by \eqref{eq:ex-Lambda-growth}; choosing the
network widths and depths as in \eqref{eq:ex-network-size} makes the abstract
neural term of order \(1/(Tn)\) at each step, which sums to the total
contribution \(1/(2n)\).

Thus \eqref{eq:ex-final} is precisely the specialization of
\eqref{eq:global-kl-power} to this example, except for one additional term:
\begin{equation}
\sum_{t=1}^T
\widetilde C_t^{\mathrm{tail}}
e^{-\widetilde c_t^{\mathrm{tail}}n^{2/7}}. \label{eq:ex-tail}
\end{equation}
This extra remainder is not a new approximation mechanism; it is the explicit
price of replacing the global compact-feature hypothesis by the growing windows
\(K_{t,n}\) from \eqref{eq:ex-Ktn}.  On each window \(K_{t,n}\), the example
fits the abstract framework exactly, and the only residual error comes from the
Gaussian mass of \(Z_t\) outside that window.  In particular, among the
\(n\)-dependent terms, the dominant one is still the algebraic
Gaussian-mixture error \(n^{-2/7}\); the neural correction \(1/(2n)\) and the
window-tail remainder are both strictly smaller as \(n\to\infty\).

\subsection{Summary}
\label{sec:ex-summary}

\subsubsection{Rate structure}

For fixed diffusion horizon \(T\), the final estimate \eqref{eq:ex-final}
decomposes into four contributions:
\begin{center}
	\renewcommand{\arraystretch}{1.2}
	\begin{tabularx}{\textwidth}{ll>{\raggedright\arraybackslash}X}
		\toprule
		\textbf{Error source} & \textbf{Total contribution} & \textbf{Behavior} \\
		\midrule
		Terminal mismatch
		& $\rho^{2T}M^2/6$
		& independent of $n$; exponentially small in $T$ \\
		
		Gaussian-mixture error
		& $\sum_{t=1}^T C_t^{\mathrm{Nor}}\,n^{-2/7}$
		& algebraic; dominant $n$-dependent term \\
		
		Neural logit error
		& $1/(2n)$
		& smaller than $n^{-2/7}$ \\
		
		Window-tail remainder
		& $\sum_{t=1}^T
		\widetilde C_t^{\mathrm{tail}}e^{-\widetilde c_t^{\mathrm{tail}}n^{2/7}}$
		& faster than any negative power of $n$ \\
		\bottomrule
	\end{tabularx}
\end{center}

Hence, for fixed \(T\) and large \(n\), the dominant \(n\)-dependent term is
\begin{equation}
\sum_{t=1}^T C_t^{\mathrm{Nor}}\,n^{-2/7}
\;\le\;
T\,\overline C^{\,\mathrm{Nor}}\,n^{-2/7},
\qquad
\overline C^{\,\mathrm{Nor}}
:=
\max_{1\le t\le T} C_t^{\mathrm{Nor}}.
\end{equation}
The only part of the bound that does not vanish as \(n\to\infty\) is the
terminal mismatch \(\rho^{2T}M^2/6\); it is controlled separately by increasing
the diffusion horizon \(T\).

\subsubsection{Assumption verification}
The following table summarizes how the abstract
hypotheses are realized in this example.

\begin{center}
	\renewcommand{\arraystretch}{1.2}
	\begin{tabularx}{\textwidth}{l c >{\raggedright\arraybackslash}X}
		\toprule
		\textbf{Assumption} & \textbf{Status} & \textbf{Notes} \\
		\midrule
		\ref{ass:feature}\ref{assItem:compactFeature}
		& \checkmark${}^*$
		& ${}^*$Not globally compact. We work on the growing windows
		$K_{t,n}$ in \eqref{eq:ex-Ktn} and carry the resulting explicit
		window-tail remainder in \eqref{eq:ex-final}. \\
		
		\ref{ass:feature}\ref{assItem:featureDisintegration}
		& \checkmark
		& The conditional feature law is
		$f_t(y\mid z)=\varphi(y;z,\sigma^2)$; see \eqref{eq:ex-ft}. \\
		
		\ref{ass:feature}\ref{assItem:featureSufficiency}
		& \checkmark
		& $r_t^\star$ depends on $(c,x_t)$ only through
		$z=\phi_t(c,x_t)=\rho x_t+\sigma^2\rho^{t-1}c$; see
		\eqref{eq:ex-feature}. \\
		
		\ref{ass:regularity}\ref{assItem:positivity}
		& \checkmark
		& Gaussian densities are strictly positive everywhere. \\
		
		\ref{ass:regularity}\ref{assItem:continuity}
		& \checkmark
		& All relevant densities are $C^\infty$. \\
		
		\ref{ass:regularity}\ref{assItem:finiteM2}
		& \checkmark
		& In fact all polynomial moments are finite. \\
		
		\ref{ass:regularity}\ref{assItem:logBound}
		& \checkmark
		& Verified in \S\ref{sec:verifyForwardAss}; one may take
		$a_{s-1}(x_s)=\frac{|x_s|+\rho^sM+1/4}{2}$ and
		$b_s(x_s,x_{s+1})=
		\frac{\rho|x_{s+1}-\rho x_s|+\rho^2/4}{2\sigma^2}$. \\
		
		\ref{ass:smooth-logodds}
		& \checkmark
		& On each \(K_{t,n}\), the exact log-odds are smooth; here we use
		$s_t=1$ and prove \(\Lambda_{t,n}^{\max}=O(n^{2/7})\); see
		\eqref{eq:ex-Lambda-growth}. \\
		
		Cor.~\ref{cor:M0-power-rate}
		& \checkmark
		& With \(\theta_t=4\), \(\xi_t=1\), one gets
		\(\beta_t=2/7\); see \eqref{eq:ex-beta-choice}. See also
		Remark~\ref{rem:ex-rate-optimization} for the limiting
		improvement \(\beta_t\to 1/2\). \\
		
		\bottomrule
	\end{tabularx}
\end{center}

	\printbibliography

\end{document}